\documentclass[12pt]{article}
\usepackage{amsmath}
\usepackage{amsthm}
\usepackage{amssymb}
\usepackage{bm}
\usepackage[round]{natbib}
\usepackage[]{graphicx}
\usepackage{epstopdf}
\usepackage{enumerate}
\usepackage{multirow}
\usepackage{url}
\usepackage{bm}
\usepackage[margin=1in]{geometry}
\newcommand{\Le}{\left(}
\newcommand{\Ri}{\right)}

\newcommand\independent{\protect\mathpalette{\protect\independenT}{\perp}}
\def\independenT#1#2{\mathrel{\rlap{$#1#2$}\mkern2mu{#1#2}}}
\newtheorem{theorem}{Theorem}
\newtheorem{corollary}{Corollary}
\newtheorem{lemma}{Lemma}

\begin{document}
\date{}
\title{Learning Gaussian Graphical Models With Fractional Marginal Pseudo-likelihood}
\author{Janne Lepp{\"a}-aho$^{1,*}$, 
Johan Pensar$^{2}$, Teemu Roos$^{1}$, and Jukka Corander$^{3}$ \\\\
\begin{small}$^{1}$HIIT / Department of Computer Science, University of Helsinki, Finland\end{small}\\
\begin{small}$^{2}$Department of Mathematics and Statistics, {\AA}bo Akademi University, Finland\end{small} \\
\begin{small}$^{3}$Department of Mathematics and Statistics, University of Helsinki, Finland\end{small} \\
\begin{small}$^*$\textit{janne.leppa-aho@helsinki.fi}\end{small}}
\maketitle

\begin{abstract}
We propose a Bayesian approximate inference method for learning the dependence structure of a Gaussian graphical model. Using pseudo-likelihood, we derive an analytical expression to approximate the marginal likelihood for an arbitrary graph structure without invoking any assumptions about decomposability. The majority of the existing methods for learning Gaussian graphical models are either restricted to decomposable graphs or require specification of a tuning parameter that may have a substantial impact on learned structures. By combining a simple sparsity inducing prior for the graph structures with a default reference prior for the model parameters, we obtain a fast and easily applicable scoring function that works well for even high-dimensional data. We demonstrate the favourable performance of our approach by large-scale comparisons against the leading methods for learning non-decomposable Gaussian graphical models. A theoretical justification for our method is provided by showing that it yields a consistent estimator of the graph structure.
\end{abstract}

\noindent%
\vfill

\newpage

\section{Introduction}
\label{sec:intro}
\subsection{Bayesian learning of Gaussian graphical models}
Gaussian graphical models provide a convenient framework for analysing conditional independence in continuous multivariate systems (\citealt{DEMPSTER72}; \citealt{WHITTAKER}; \citealt{LAURITZEN}). We consider the problem of learning Gaussian graphical models from data using a Bayesian approach. Most of the Bayesian methods for learning Gaussian graphical models make the assumption about the decomposability of the underlying graph (\citealt{JONES05b}; \citealt{SCOTT08}; \citealt{CARVALHO09}). Recently, \citet{FITCH14} investigated how Bayesian methods assuming decomposability perform in model selection when the true underlying model is non-decomposable. Bayesian methods that do not assume decomposability have been considered more seldom in the literature, and in particular not in the high-dimensional case (\citealt{WONG03}; \citealt{ATAY05}; \citealt{MOGHADDAM09}; \citealt{DOBRA11}). 

A widely used non-Bayesian method for learning Gaussian graphical models is the graphical lasso (\citealt{GLASSO1}; \citealt{GLASSO2}). Graphical lasso (\texttt{glasso}) uses $l_1$-penalized Gaussian log-likelihood to estimate the inverse covariance matrices and does not rely on the assumption of decomposability. Other approaches include a neighbourhood selection (\texttt{NBS}) method by \cite{MB06} and Sparse Partial Correlation Estimation method (\texttt{space}) by \cite{SPACE09}. The \texttt{NBS}-method estimates the graphical structure by performing independent lasso regressions for each variable to find the estimates for the neighbourhoods whereas \texttt{space} imposes an $l_1$-penalty on an objective function corresponding to an $l_2$-loss of a regression problem in order to estimate the non-zero partial correlations which correspond to edges in the graphical model. 

Assuming decomposability in Bayesian methods has been popular, since it enables derivation of a closed form expression for the marginal likelihood under a conjugate prior. In our approach we bypass this restriction by replacing the true likelihood in the marginal likelihood integral by a pseudo-likelihood. This implies a factorization of the marginal pseudo-likelihood into terms that can be evaluated in closed form by using existing results for the marginal likelihoods of Gaussian directed acyclic graphs. The marginal pseudo-likelihood offers further advantages by allowing efficient search algorithms to be used, such that model optimization becomes realistic for even high-dimensional data. \cite{DOBRA04} considered a similar pseudo-likelihood based approach. These two methods involve similar techniques in the first step where a general dependency network is learned using a Bayesian approach. However, in the method by \citeauthor{DOBRA04}, the learned, undirected network is converted  into a directed acyclic graph in order to obtain a proper joint density. We use pseudo-likelihood only in order to develop a scoring function which can be used to perform efficient structure learning of Gaussian graphical models and we do not need to estimate the joint density at any stage.

Marginal pseudo-likelihood has been previously used to learn undirected graphical models with discrete variables in \cite{MPL}. Our paper can be seen to generalize the ideas developed there to the continuous domain by introducing the required methodology and providing a formal consistency proof under the multivariate normal assumption. Our method utilizes the fractional Bayes factors based approach of \cite{CONSONNI12} to cope automatically with the difficulty of setting up prior distributions for the models' parameters.

The rest of the paper is organized as follows. After introducing the notation, we briefly review the results by \citeauthor{CONSONNI12} that are needed in deriving the expression for the marginal pseudo-likelihood. In Section 3 we state our main result by introducing the fractional marginal pseudo-likelihood. The detailed proof of its consistency for Markov blanket estimation is given in Appendix. A score-based search algorithm adopted from \cite{MPL} is presented in order to implement the method in practice. In Section 4 we demonstrate the favourable  performance of our method by several numerical experiments involving a comparison against \texttt{glasso}, \texttt{NBS} and \texttt{space}. 

\subsection{Notations and preliminaries}
We will start by reviewing some of the basic concepts related to graphical models and the multivariate normal distribution. For a more comprehensive presentation, see for instance \citep{WHITTAKER} and \citep{LAURITZEN}. 

Consider an undirected graph $G = (V,E)$, where $V=\{ 1, \ldots, p \}$ is the set of nodes (vertices) and $E \subset V \times V$ is the set of edges. There exists an (undirected) edge between the nodes $i$ and $j$, if and only if $(i,j)\in E$ and $(j,i) \in E$. Each node of the graph corresponds to a random variable, and together they form a $p$-dimensional random vector $\bm{x}$. We will use the terms node and variable interchangeably. Absence of an edge in  the graph $G$ is a statement of conditional independence between the corresponding elements of $\bm{x}$. More in detail, $(i,j) \notin E$ if and only if $x_i$ and $x_j$ are conditionally independent given the remaining variables $\bm{x}_{V \setminus \{ i,j \}}$. This condition is usually referred as the pairwise Markov property. We let $\mathit{mb(j)}$ denote the Markov blanket of node $j$. The Markov blanket is defined as the set containing the neighbouring nodes of $j$, $\mathit{mb(j)} = \{ i \in V \ | \ (i,j) \in E \}$. The local Markov property states that each variable is conditionally independent of all others given its Markov blanket. An undirected graph $G$ is called decomposable or equivalently chordal if each cycle, whose length is greater than $4$, contains a chord. By a cycle, we mean a sequence of nodes such that the subsequent nodes are connected by an edge and the starting node equals the last node in the sequence. The length of a cycle equals the number of edges in the cycle. A chord is an edge between two non-subsequent nodes of the cycle.

We will write $\bm{x} \sim N_p(\bm{0}, \bm{\Omega}^{-1})$ to state that a random vector $\bm{x}$ follows a $p$-variate normal distribution with a zero mean and precision matrix $\bm{\Omega}$. We will denote the covariance matrix by $\bm{\Sigma} = \bm{\Omega}^{-1}$. The precision matrix $\bm{\Omega}$, and also equivalently $\bm{\Sigma}$, are always assumed to be symmetric and positive definite. 

Given an undirected graph $G$ and a random vector $\bm{x}$, we define a Gaussian graphical model to be the collection of multivariate normal distributions for $\bm{x}$ that satisfy the conditional independences implied by the graph $G$. Hence, a Gaussian graphical model consists of all the distributions $N_p(\bm{0},\bm{\Omega}^{-1})$, where $\bm{\Omega}_{ij} = 0$ if and only if $(i,j) \notin E$, $i\neq j$. Otherwise, the elements of the inverse covariance matrix can be arbitrary, as long as symmetry and positive definiteness holds.

In contrast to the above undirected model, a Gaussian directed acyclic graphical model is a collection of multivariate normal distributions for $\bm{x}$, whose independence structure can be represented by some directed acyclic graph (DAG) $D = (V,E)$. When considering directed graphs, we use $\mathit{pa(j)}$ to denote the parent set of the node $j$. The set $\mathit{pa(j)}$ contains nodes $i$ such that $(i,j) \in E$. That is, there exists a directed edge from $i$ to $j$. Similar Markov assumptions as those characterizing the dependency structure under undirected models, as described above, hold also for directed models, see, for instance, \citep{LAURITZEN}. For each decomposable undirected graph, we can find a DAG which defines the same conditional independence assertions. In general, the assertions representable by DAGs and undirected graphs are different.

\section{Objective Comparison of Gaussian Directed Acyclic Graphs}
\cite{CONSONNI12} consider objective comparison of Gaussian directed acyclic graphical models and present a convenient expression for computing marginal likelihoods for any Gaussian DAG. Their approach to Gaussian DAG model comparison is based on using Bayes factors and uninformative, typically improper prior on the space of unconstrained covariance matrices. Ambiguity arising from the use of improper priors is dealt with by utilizing the fractional Bayes factors (\citealt{OHAGAN95}). 

We first review a result concerning the computation of marginal likelihood in a more general setting, presented by \cite{GEIGER02}. They state five assumptions concerning the regularity of the sampling distribution of data and the structure of the prior distribution for parameters, that allow construction of parameter priors for every DAG model with a given set of nodes by specifying only one parameter prior for any of the complete DAG models. A complete DAG model refers to a model in which every pair of nodes is connected by an edge, implying that there are no conditional independence assertions between the variables. When the regularity assumptions are met, the following result can be derived:

\begin{theorem}(Theorem 2 in \citeauthor{GEIGER02})\label{genDAG}
Let $M$ and $M_c$ be any DAG model and any complete DAG model for $\bm{x}$, respectively. Let $\bm{X}$ denote a complete (no missing observations) random sample of size $n$. Now the marginal likelihood for $M$ is 
\begin{equation}\label{genDAG2}
p(\bm{X}\mid M) = \prod_{j = 1}^{p}\frac{p(\bm{X}_{\mathit{fa(j)}}\mid M_c)}{p(\bm{X}_{\mathit{pa(j)}}\mid M_c)},
\end{equation} where $\bm{X}_{\mathit{pa(j)}}$ denotes the data belonging to the parents of $x_j$. We call $\mathit{\mathit{fa(j)}}= \mathit{pa(j)} \cup \{j\}$ the family of variable $x_j$.  
\end{theorem}
Assumptions given by \citeauthor{GEIGER02} also imply that the marginal likelihood given by (\ref{genDAG2}) scores all Markov equivalent DAGs equally, which is a desirable property when DAGs are considered only as models of conditional independence. 

In order to apply (\ref{genDAG2}), \citeauthor{CONSONNI12} derive expressions for the marginal likelihoods corresponding to subvectors of $\bm{x}$, given the complete Gaussian DAG model. Objectivity is achieved by using an uninformative improper prior of the form
\begin{equation}\label{improprior}
p(\bm{\Omega}) \propto |\bm{\Omega}|^{\frac{a_{\bm{\Omega}}-p-1}{2}},
\end{equation} for the parameters of the complete DAG model. The improper prior is updated into a proper one by using fractional Bayes factors approach \citep{OHAGAN95}. In this approach, a fraction of likelihood is ``sacrificed'' and used to update the improper prior into a proper fractional prior which is then paired with the remaining likelihood to compute the Bayes factors. \citeauthor{CONSONNI12} show that the resulting fractional prior on the precision matrix $\bm{\Omega}$ is Wishart. This choice of prior combined with Gaussian likelihood satisfies all five assumptions required to use (\ref{genDAG2}).  

Setting $a_{\bm{\Omega}} = p -1$, we can take the fraction of sacrificed likelihood to be $1/n$, see ( \citealt{CONSONNI12}). Now applying (\ref{genDAG2}) and Eq. (25) in \citeauthor{CONSONNI12}, we obtain the marginal likelihood of any Gaussian DAG as
\begin{eqnarray}\label{finalML}
p(\bm{X} \mid M) =   \prod_{j=1}^p \pi^{-\frac{(n-1)}{2}} \frac{\Gamma \Le \frac{n+p_j}{2} \Ri}{ \Gamma \Le \frac{p_j+1}{2} \Ri} 
 n^{-\frac{2p_j+1}{2}}\Le \frac{|\bm{S}_{\mathit{fa(j)}}|}{|\bm{S}_{\mathit{pa(j)}}|}\Ri^{-\frac{n-1}{2}},
\end{eqnarray} where $p_j$ is the size of the set $\mathit{pa(j)}$, $\bm{S} = \bm{X}^T \bm{X}$ is the unscaled sample covariance matrix and $\bm{S}_A$ refers to a submatrix of $\bm{S}$ restricted to variables in the set $A$. The fractional marginal likelihood given by (\ref{finalML}) is well defined if matrices $\bm{S}_{\mathit{pa(j)}}$ and $\bm{S}_{\mathit{fa(j)}}$ are positive definite for every $j$. This is satisfied with probability $1$ if $n \geq \max \{ p_j + 1 \ | \ j = 1, \ldots , p  \}$. 

\citeauthor{CONSONNI12} also show that their methodology can be used to perform model selection among decomposable Gaussian graphical models. This is possible because every decomposable undirected graph is Markov equivalent to some DAG. A similar fractional marginal likelihood approach as presented above has been considered by \cite{CARVALHO09} in the context of decomposable Gaussian graphical models. 

\section{Structure Learning of Gaussian Graphical Models}\label{THEORY}
\subsection{Marginal Likelihood}

Suppose we have a sample of independent and identically distributed multivariate normal data $\bm{X} = (\bm{X}_1, \ldots, \bm{X}_n)^T$, coming from a distribution whose conditional dependence structure is represented by an undirected graph $G^*$. We aim at identifying $G^*$ based on $\bm{X}$, which is done with a Bayesian approach by maximizing the approximate posterior probability of the graph conditional on the data.  

Posterior probability of a graph $G$ given data $\bm{X}$ is proportional to
\begin{equation}\label{posterior}
p(G \mid \bm{X}) \propto p(G)p(\bm{X}\mid G),
\end{equation} where $p(G)$ is the prior probability assigned to a specific graph and $p(\bm{X}\mid G)$ is the marginal likelihood. The normalizing constant of the posterior can be ignored, since it cancels in comparisons of different graphs. First, we focus on the marginal likelihood, since it is the data dependent term in (\ref{posterior}). Later on, we will make use of the prior $p(G)$ term in order to promote sparseness in the graph structure.

By definition, the marginal likelihood of $\bm{X}$ under $G$ equals
\begin{equation}\label{MARGLIKE5}
p(\bm{X}\mid G) = \int_{\Theta_G} p(\bm{\theta} \mid G) p(\bm{X} \mid \bm{\theta}, G) d\bm{\theta} , 
\end{equation} where $\bm{\theta}$ is the parameter vector, $p(\bm{\theta} \mid G)$ denotes the parameter prior under $G$, the term $p(\bm{X} \mid \bm{\theta}, G)$ is the likelihood function and the integral is taken over the set of all possible parameters under $G$. 

However, computing the marginal likelihood for a general undirected graph is very difficult, due the global normalizing constant in the likelihood term. Closed form solution exists only for chordal graphs, which is a highly restrictive assumption in general.

\subsection{Marginal Pseudo-likelihood}
We circumvent the problem of an intractable integration involved with the true likelihood function by using pseudo-likelihood. Pseudo-likelihood was introduced originally by \cite{BESAG1972}. The idea behind the pseudo-likelihood can be motivated by thinking of it as an approximation for the true likelihood in form of a product of conditional probabilities or densities, where in each factor the considered variable is conditioned on all the rest. More formally, we write the pseudo-likelihood as
$$
\hat{p}(\bm{X} \mid  \bm{\theta}) = \prod_{j = 1}^p p(X_j \mid \bm{X}_{-j}, \bm{\theta}),
$$ where the notation $\bm{X}_{-j}$ stands for observed data on every variable except the $j$:th one.

In general, pseudo-likelihood should not be considered as a numerically exact and accurate approximation of the likelihood but as an object that has a computationally more attractive form and which can be used to obtain consistent estimates of parameters. It can be shown that under certain regularity assumptions, the pseudo-likelihood estimates for model parameters coincides with the maximum likelihood estimates, see \cite{KOLLER}.

One advantage of using pseudo-likelihood instead of the true likelihood is that it allows us to replace the global normalization constant by $p$ local normalising constants related to conditional distributions of variables and thus makes the computations more tractable. 
  
Using pseudo-likelihood, the original problem (\ref{MARGLIKE5}) of computing the marginal likelihood of $\bm{X}$ can be stated as
\begin{eqnarray}
\label{MARGLIKE52}
p(\bm{X}\mid G) &\approx & \int_{\Theta_G} p(\bm{\theta} \mid G) \prod_{j=1}^p p(\bm{X}_j\mid \bm{X}_{-j}, \bm{\theta} , G) d\bm{\theta} \\
&=&\hat{p}(\bm{X}\mid G)
\end{eqnarray}
The term $\hat{p}(\bm{X}\mid G)$ is referred to as the marginal pseudo-likelihood, introduced by \cite{MPL} for discrete-valued undirected graphical models. The local Markov property states that given the variables in its Markov blanket $\mathit{mb(j)}$, the variable $x_j$ is conditionally independent of the remaining variables. More formally, we have that
$$
p(x_j \mid \bm{x}_{-j}, \bm{\theta} ) = p(x_j \mid \bm{x}_{\mathit{mb(j)}}, \bm{\theta}).
$$Thus, we obtain the following form for the marginal pseudo-likelihood
\begin{equation}\label{MPL1}
\hat{p}(\bm{X}\mid G) = \int_{\Theta_G} p(\bm{\theta} \mid G) \prod_{j=1}^p p(\bm{X}_j\mid \bm{X}_{\mathit{mb(j)}}, \bm{\theta}) d\bm{\theta}
\end{equation} We assume global parameter independence in order to factor the full integral into integrals over individual parameter sets $\Theta_j$ related to conditional distributions $p(x_j \mid \bm{x}_{\mathit{mb(j)}}).$ The expression for the integral (\ref{MPL1}) becomes
\begin{equation}\label{MPL2}
\hat{p}(\bm{X}\mid G)=\prod_{j=1}^p \int_{\Theta_j} p(\bm{\theta}_j)p(\bm{X}_j\mid \bm{X}_{\mathit{mb(j)}}, \bm{\theta}_j ) d\bm{\theta}_j .
\end{equation} 

\subsection{Fractional Marginal Pseudo-likelihood}
The expression (\ref{MPL2}) for the marginal pseudo-likelihood can be regarded as a product of terms, where each term corresponds to a marginal likelihood of a DAG model. This offers a tractable way to compute the marginal pseudo-likelihood in closed form. 

Recall the general formula for a marginal likelihood of any DAG model $M$, introduced in the previous section:
\begin{equation}\label{margLK}
p(\bm{X}\mid M) = \prod_{j = 1}^{p}\frac{p(\bm{X}_{\mathit{fa(j)}}\mid M_c)}{p(\bm{X}_{\mathit{pa(j)}}\mid M_c)} = \prod_{j = 1}^{p} p(\bm{X}_j \mid \bm{X}_{\mathit{pa(j)}}, M_c),
\end{equation} where in the last equality we used the definition $\mathit{fa(j)} = \{ j \} \cup \mathit{pa(j)}.$ 

We can see a clear resemblance between the forms (\ref{margLK}) and (\ref{MPL2}). In both of these, each factor corresponds to a marginal likelihood of a DAG model, where we have a node and its parent nodes. In the case of Markov networks, the set of parents of a node is its Markov blanket, $\mathit{mb(j)}$. 

Thus, we can use the closed form solution of (\ref{finalML}) to compute the sought marginal pseudo-likelihood (\ref{MPL2}) by changing $\mathit{pa(j)} \to \mathit{mb(j)}$ and defining $\mathit{fa(j)} = \{ j \} \cup \mathit{mb(j)}.$ Then the closed form solution (\ref{finalML}) for the fractional likelihood corresponds to
\begin{eqnarray}\label{FMPL1}
\hat{p}(\bm{X}\mid G) &=& \prod_{j = 1}^{p}\pi^{-\frac{(n-1)}{2}} \frac{\Gamma \Le \frac{n+p_j}{2} \Ri}{ \Gamma \Le \frac{p_j+1}{2} \Ri}
 n^{-\frac{2p_j+1}{2}}\Le \frac{|\bm{S}_{\mathit{fa(j)}}|}{|\bm{S}_{\mathit{mb(j)}}|}\Ri^{-\frac{n-1}{2}}\nonumber\\
&= &\prod_{j = 1}^{p} p(\bm{X}_j \mid \bm{X}_{\mathit{mb(j)}}),
\end{eqnarray} where $p_j= |\mathit{mb(j)}|$ and $\bm{S}$ refers to the full $p\times p$ unscaled sample covariance matrix. As before, $\bm{S}_{\mathit{mb(j)}}$ and $\bm{S}_{\mathit{fa(j)}}$ refer to submatrices of $\bm{S}$ restricted to variables in sets $\mathit{mb(j)}$ and $\mathit{fa(j)}$. From now on, $\hat{p}(\bm{X} \mid G)$ is referred to as fractional marginal pseudo-likelihood, due to the fractional Bayes factor approach used in derivation of the analytical form. The expression $p(\bm{X}_j \mid \bm{X}_{\mathit{mb(j)}})$ is used to denote the local fractional marginal pseudo-likelihood for the node $j$.

The next theorem provides a theoretical justification for the approximation used in derivation of our scoring criterion. 

\begin{theorem}\label{consistency}
Let $\bm{x} \sim N_p(\bm{0}, (\bm{\Omega}^*) ^{-1})$ and $G^* = (V,E^*)$ denote the the undirected graph that completely determines the conditional independence statements between $\bm{x}$'s components. Let $\{ mb^*(1), \ldots , mb^*(p)\}$ denote the set of Markov blankets, which uniquely define $G^*.$ 

Suppose we have a complete random sample $\bm{\bm{X}}$ of size $n$ obtained from $N_p(\bm{0}, (\bm{\Omega}^*) ^{-1}).$ Then for every $j\in V$, the local fractional marginal pseudo-likelihood estimator 
$$
\widehat{mb}(j) = {\arg\max}_{\mathit{mb(j)} \subset V\setminus \{j\}} p(\bm{X}_j \mid \bm{X}_{\mathit{mb(j)}})
$$ is consistent, that is, $\widehat{mb}(j) = mb^*(j)$ with probability tending to $1$, as $n\to \infty .$
\end{theorem}

The detailed proof of Theorem \ref{consistency} is presented in Appendix A. The proof is split in two parts; first, we show that the fractional marginal pseudo-likelihood score does not overestimate, \textit{i.e.}, the true Markov blanket is preferred over the sets containing redundant nodes. The second part covers the underestimation: a set that does not contain all the members of the true Markov blanket will receive strictly lower score. Combining these two results implies our theorem. The strategy of dividing a proof in these kinds of cases is fairly common approach when proving the consistency of model selection criteria, see, for instance, (\citealt{haugh88}, \citealt{wei92}). The essential part in our proof is studying the asymptotic form of the data dependent term and showing that it behaves as desired in both of the required cases. The statements proven can be formulated into following lemmas:

\begin{lemma}\label{overLemma}
Overestimation. Let $\mathit{mb}^* \subset V\setminus \{ j \}$ and $fa^* = \mathit{mb}^* \cup \{ j \} $ denote the true Markov blanket and the true family of the node $j\in V$, respectively. Let $\mathit{mb} \subset V\setminus \{ j \}$ be a superset of the true Markov blanket, $\mathit{mb}^* \subset \mathit{mb}$. Now, as the sample size $n\to\infty$
$$
\log\frac{p(\textbf{X}_j\ \mid\ \textbf{X}_{\mathit{mb}^*})}{p(\textbf{X}_j\ \mid\ \textbf{X}_{\mathit{mb}})} \to \infty
$$ in probability.   
\end{lemma}
\begin{lemma}\label{underLemma}
Underestimation. Let $\mathit{mb}^* \subset V\setminus \{ j \}$ and $fa^* = \mathit{mb}^* \cup \{ j \} $ denote the true Markov blanket and the true family of the node $j \in V$, respectively. Assume that $\mathit{mb} \subset \mathit{mb}^*$. Let $A \subset V\setminus \mathit{fa}^* $. Now, as the sample size $n\to\infty$
$$\log\frac{p(\bm{X}_j\ |\ \bm{X}_{\mathit{mb}^*\cup A})}{p(\bm{X}_j\ |\ \bm{X}_{\mathit{mb}\cup A})} \to \infty
$$ in probability.
\end{lemma}

In Lemma \ref{underLemma}, we also allow for cases where $\mathit{mb} = \emptyset $ or $A = \emptyset$. With these proven, it is easy to see that our scoring function will asymptotically prefer the true Markov blanket over any other possible Markov blanket candidate. For supersets of the true Markov blanket, this follows from the overestimation lemma. For an arbitrary set that does not contain all the true members, we can apply the underestimation lemma to show that there is always a set with strictly higher score. This set is either the true Markov blanket or its superset. This suffices, since the latter case reduces to using the overestimation lemma again.     

To be a bit more specific, consider a set $\mathit{mb}$ which has the same cardinality as the true Markov blanket but does not contain all the true nodes. This set is not a superset, nor a subset of the true Markov blanket but it will receive a lower score asymptotically. This follows, since the underestimation lemma guarantees that a set that contains all the members of the true Markov blanket and the redundant ones from $\mathit{mb}$, will be preferred over mere $\mathit{mb}$. This reduces the problem to comparing the score of the true Markov blanket with its superset which is covered by the overestimation part.           

The locally consistent Markov blankets imply that the whole graph is also asymptotically correctly estimated which is formulated in the following corollary:    

\begin{corollary} Let $\mathcal{G}$ denote the set of all undirected graphs with $p$ nodes. The global fractional marginal pseudo-likelihood estimator
$$
\widehat{G} = {\arg\max}_{G \in \mathcal{G}} \ \hat{p}(\bm{X} \mid G)
$$ is consistent, that is, $\widehat{G} = G^*$ with probability tending to $1$, as $n\to \infty .$
\end{corollary} 
\begin{proof}
Theorem \ref{consistency} guarantees that the true Markov blanket of each node is found with a probability tending to $1$ as sample size increases. Since the structure of a Markov network is uniquely determined by its Markov blankets, the result follows. 
\end{proof}
\subsection{Learning Algorithm for Markov Blanket Discovery}
The consistency result of the local Markov blanket estimators stated in the last section allows us to optimise the Markov blanket of each variable independently. In practice, this is done by implementing a greedy hill-climb algorithm (Algorithm 1 in \citealt{MPL}) with the fractional marginal pseudo-likelihood as a scoring function. However, as the consistency is an asymptotic result, we are not guaranteed to produce proper undirected graphs on small sample sizes. To be more specific, we may find Markov blankets $mb(i)$ and $\mathit{mb(j)}$ such that $i\in \mathit{mb(j)}$ but $j \not\in \mathit{mb(j)}$, which contradicts the definition of an undirected graph. To overcome this, we use two criteria, \texttt{AND} and \texttt{OR}, to combine the learned Markov blankets into proper undirected graphs. 

Denote the identified Markov blankets by $\mathit{mb(j)}, \ j = 1, \ldots , p$. The edge sets specifying \texttt{OR}- and \texttt{AND}-graphs are correspondingly defined as follows 
\begin{eqnarray*}
E_{\texttt{OR}} &= \{ (i,j) \subset V \times V \ | \ i \in \mathit{mb(j)} \textnormal{ or } j\in mb(i) \} \\
E_{\texttt{AND}} &= \{ (i,j) \subset V \times V \ | \ i \in \mathit{mb(j)} \textnormal{ and } j\in mb(i) \}. 
\end{eqnarray*} 

In addition to \texttt{AND}- and \texttt{OR}-method we consider a third procedure referred to as the \texttt{HC}-method (Algorithm 2 in \citealt{MPL}). Starting point for the \texttt{HC}-method is the graph obtained by \texttt{OR}-method which is used to define a subspace of graphs $\mathcal{G}_{\texttt{OR}} = \{ G \in \mathcal{G} | \ E \subset E_{\texttt{OR}} \}$. Then a simple deterministic greedy hill-climb is performed in the reduced model space $\mathcal{G}_{\texttt{OR}}$  by removing or adding single edges resulting in the largest improvement in the fractional marginal pseudo-likelihood score. 

\subsection{Sparsity Promoting Prior Over Local Graphs}
Until now we have assumed that every graph structure is \textit{a priori} equally likely and thus the prior term $p(G)$ in (\ref{posterior}) was ignored. However, in most applications with high-dimensional variable sets it is natural to assume that the underlying dependence structure is sparse. To promote sparsity beyond the basic Occham's razor, which is built into Bayesian model comparison, one can use the prior distribution $p(G)$ to penalize nodes for having too many elements in their Markov blankets. By defining our graph prior in terms of mutually independent prior beliefs about the Markov blankets, we maintain the useful factorization of our score and the local score is given by 
\begin{equation}
p(\mathit{mb(j)})p(\bm{X}_j \mid \bm{X}_{\mathit{mb(j)}}).
\end{equation} We start with a similar approach as used for example in \cite{CARVALHO09} to motivate our choice for the prior. In this approach, we assume that the inclusion of an edge in a graph happens with some unknown probability $t$, which corresponds to a successful Bernoulli trial. A finite sequence of these inclusions is a repeated Bernoulli trial and thus binomially distributed. We obtain the following form for the local prior
\begin{equation}\label{PRIORI}
p(\mathit{mb(j)} \mid  t ) \propto t^{p_j}(1-t)^{m-p_j},
\end{equation} where $p_j$ is the proposed size of the Markov blanket of $j$, or equivalently the number of edges connected to $j$ (number of successes in repeated Bernoulli trials). We use $m$ to represent the maximum number of edges, that could be present in a local graph, that has $p_j + 1$ nodes. Hence $m$ corresponds to the number of trials. Strictly speaking, such an interpretation is somewhat misleading since $p_j$ can be at most $p-1$ and $m = p_j(p_j +1)/2$ depends on it. Nevertheless, this approach defines a proper prior since the prior scores derived from equation (\ref{PRIORI}) can be normalized by a constant that depends only on $p$, and thus cancels when comparing local graph structures. This prior is shown to perform favourably in the numerical tests considered later.

An appropriate value for the parameter $t$ would be unknown for most applications. To overcome this issue, we put a prior on the parameter and integrate it out to obtain a suitable prior score function. Choosing a conjugate prior $t \sim \textnormal{Beta}(a,b)$ and integrating leads to the expression
\begin{equation}\label{PRIOR}
p(\mathit{mb(j)}) \propto  \frac{\beta(a + p_j,b + m - p_j)}{\beta(a,b)},
\end{equation} where $\beta(\cdot , \cdot)$ refers to the beta function. In our numerical experiments, we use $a = b = 1/2$. Motivation for this choice is that $\textnormal{Beta}(1/2,1/2)$ is the Jeffreys' prior for the probability parameter of the binomial distribution, see, for instance, \citep{BAYESDATA}.

\section{Numerical Experiments}\label{EXPERIMENTS}

\subsection{Structure Learning with Synthetic Data}
We first study the performance of the fractional marginal pseudo-likelihood in learning the graphical structures from synthetic multivariate normal data. We specify the structure of the generating network and measure the quality of the learned graphs using the Hamming distance which is defined as the number of edges to be added and deleted from a learned graph to obtain the true generating graph structure.

The synthetic graphs used to create our data are constructed by using 4 different subgraphs as building blocks. Graphs are shown in the Figure \ref{fig:graphs}. Subgraphs are combined together as disconnected components to create a 64 node graph. This graph is again used as a component to build larger graphs. In total, the dimensions in the sequence $p = 64, 128, 256, 512, 1024 $ are considered.

\begin{figure}
\centering
\includegraphics[height= .22\linewidth]{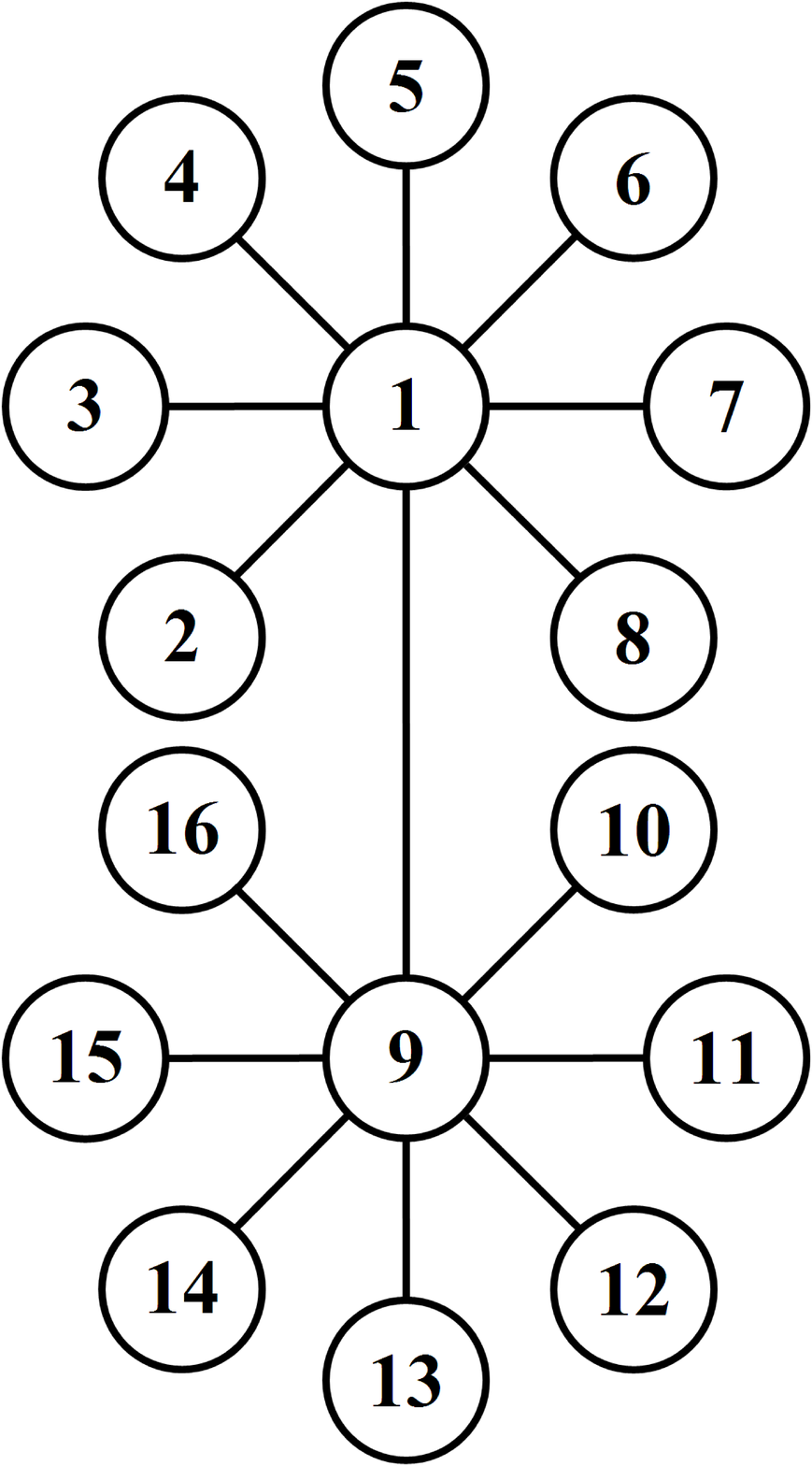} \phantom{.......}
\includegraphics[height=.22\linewidth]{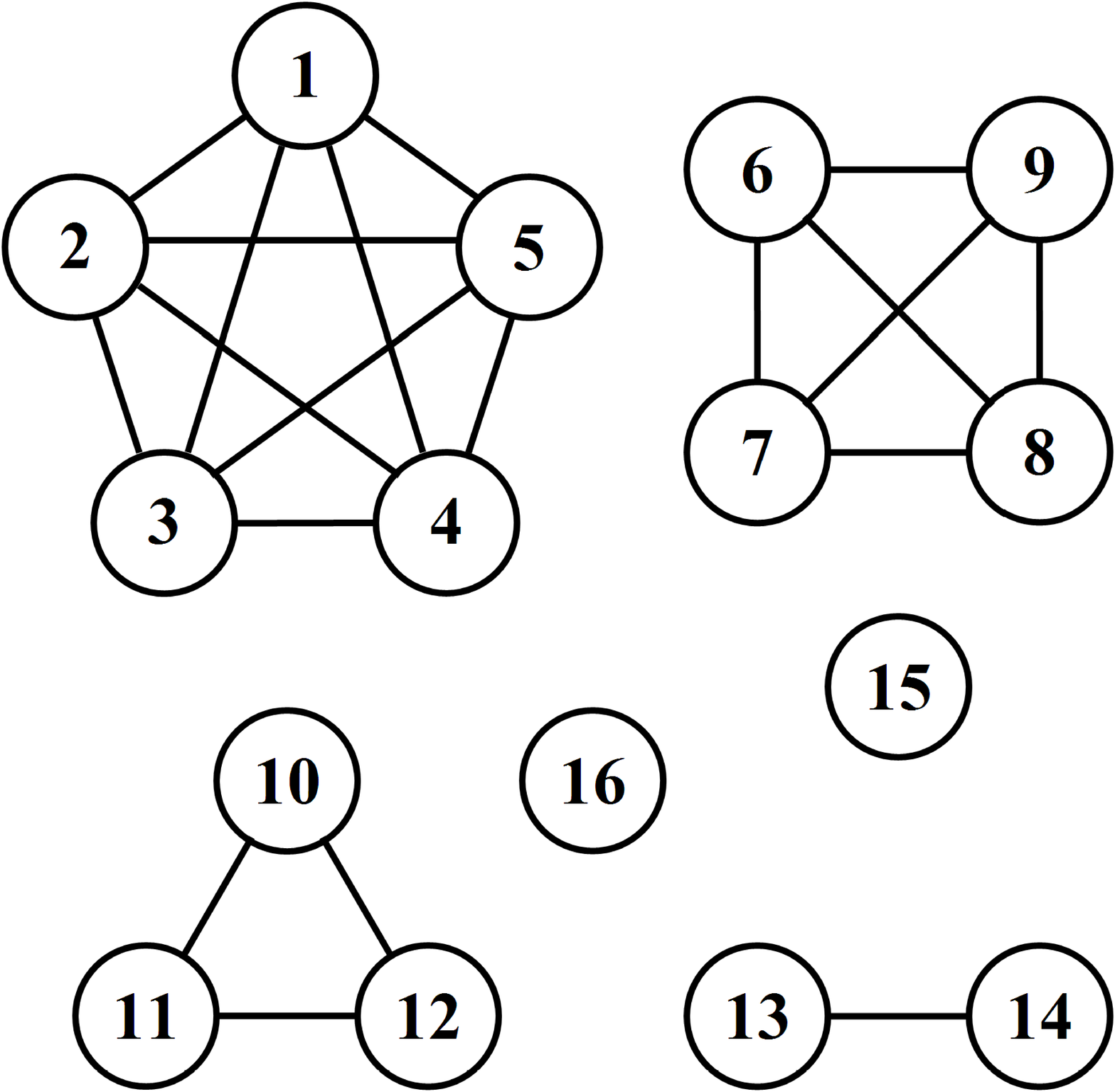} \phantom{.......}
\includegraphics[height= .22\linewidth]{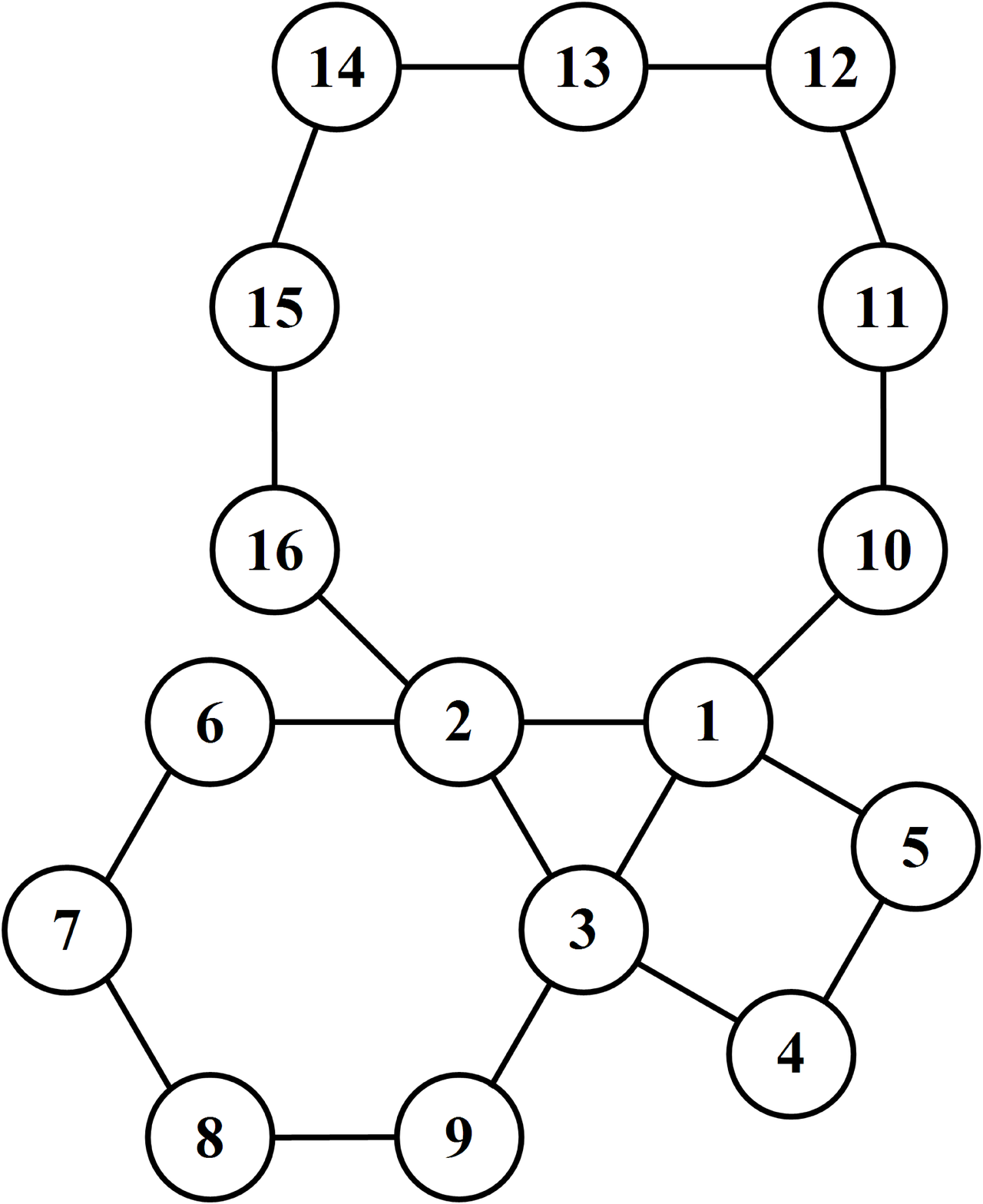} \phantom{.......}
\includegraphics[height= .22\linewidth]{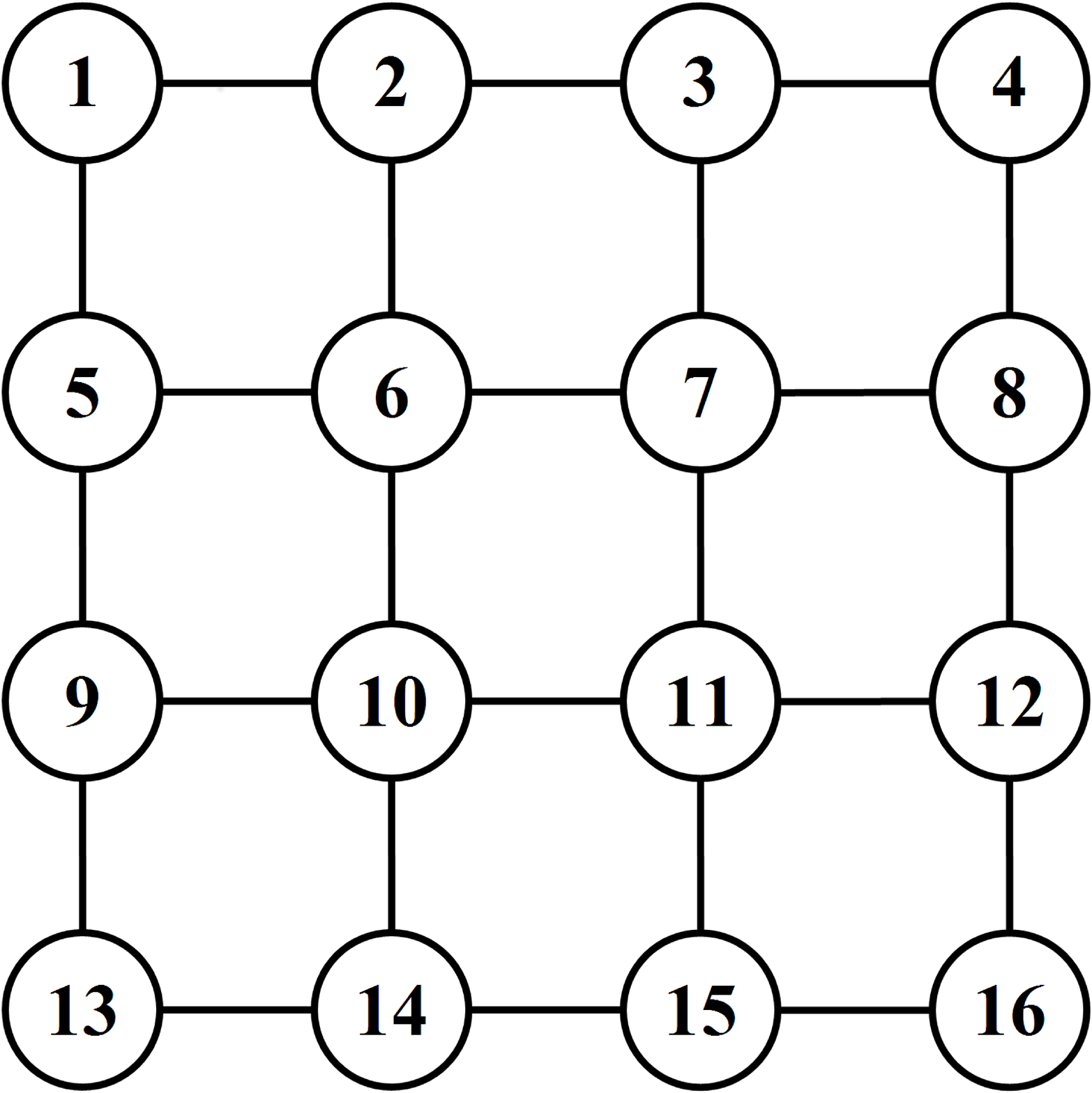}\caption{Synthetic subgraphs. Pictures appear originally in \protect\cite{MPL}.} 
\label{fig:graphs}
\end{figure}

When the graph structure is specified, we construct the corresponding precision matrices by setting elements to zeros as implied by the graph. The absolute values of the remaining off-diagonal elements are chosen randomly between $0.1$ and $0.9$ so that about half of the elements are negative. The diagonal elements are also first chosen randomly from the same interval and then a suitable vector is added to the diagonal in order to make all the eigenvalues positive, thus ascertaining the positive definiteness of the precision matrix. Finally, the matrix is inverted to get the covariance matrix and zero mean multivariate normal data is sampled using the built-in function 'mvnrnd' in Matlab.

For each of the considered dimensions, we created 25 covariance matrices, sampled a data set of 4000 observations and learned the structures using fractional marginal pseudo-likelihood, \texttt{glasso}, \texttt{space} and \texttt{NBS} with different sample sizes. Input data were scaled so that each variable had zero mean and a standard deviation of one. The sparsity promoting prior was used with the fractional marginal pseudo-likelihood methods.

\texttt{Glasso} requires a user-specified tuning parameter that affects the sparsity of the estimated precision matrix. For every input data, we computed \texttt{glasso} using 12 different values for the tuning parameter logarithmically spaced on the interval $[0.01 , 1]$. The best value for $\lambda$ was chosen according to the extended BIC criterion proposed by \cite{EBIC}:
$$
\texttt{EBIC}(\lambda) = n \phantom{.}\textnormal{tr}(\hat{\bm{\Omega}}\bm{C}) - n \log\det (\hat{\bm{\Omega}}) + K \log n + 4 K \gamma \log p ,
$$ where $n$ denotes sample size, $p$ is the number of variables, $\bm{C} = (1/n)\textbf{S}$ is the maximum likelihood estimate for the covariance matrix, $\hat{\bm{\Omega}}$ stands for the estimate of the inverse covariance for given $\lambda$ and $K$ is the number of non-zero elements in the upper-half of $\hat{\bm{\Omega}}$, that is the number of edges in the corresponding graphical model. The parameter $\gamma$ is constrained to be between $0$ and $1$. By using the value $\gamma = 0$, we would retain the ordinary BIC criterion, and increasing the $\gamma$ would encourage sparser solutions. In the experiments, we used the value $\gamma = 0.5$.

The parameter value $\lambda$ minimising the above criterion was used and the graphical model was read from the corresponding estimate of $\hat{\bm{\Omega}}$. R-package '\texttt{glasso}' \citep{GLASSOCODES} was used to perform the computations for \texttt{glasso}.

The computations for \texttt{NBS} were carried out using the Meinshausen-B\"{u}hlmann approximation also implemented in the R-package `\texttt{glasso}'. The required tuning parameter $\lambda$  was chosen automatically, as proposed by the authors \citep{MB06} to be $\lambda = (n^{-1/2})\Phi^{-1}(1- \alpha/ (2p^2)),$ where $\alpha = 0.05$ and $\Phi(\cdot)$ denotes the c.d.f. of a standard normal random variable. Parameter $\alpha$ is related to the probability of falsely connecting two separate connectivity components of the true graph, see ch. 3 in \cite{MB06}. Since the resulting inverse covariance matrix, $\hat{\bm{\Omega}}$, was not necessarily symmetric, we used the average $(1/2)(\hat{\bm{\Omega}} + \hat{\bm{\Omega}}^T)$ to determine the estimated graph structure for \texttt{NBS}.   

For the computations of \texttt{space} we used the corresponding R-package \citep{SPACECODES}. Also for this method, the user is required to specify a tuning parameter $\lambda$ controlling the $l_1$-regularisation. We selected the scale of the tuning parameter to be $s = (n^{1/2})\Phi^{-1}(1- \alpha/ (2p^2))$ with $\alpha = 0.05$. Twelve candidate values for the tuning parameter were then chosen by multiplying a vector of $12$ linearly space numbers from $0.5$ to $3.25$ by the scaling constant $s$. The best value for the $\lambda$ was then chosen according to the BIC styled criterion proposed by the authors of the method (see ch. 2.4 in \citealt{SPACE09}). The \texttt{space} algorithm was run with uniform weights for regressions in the joint loss function and iteration parameter set to $2$. For both \texttt{glasso} and \texttt{space} the range of possible tuning parameters was selected so that the best value according to the used criterion would lie strictly inside the given grid in all of the tests.

The Hamming distance results for the structure learning tests are shown in Figures \ref{fig:HDplots} and \ref{fig:HDplots2}. For the sake of clarity, \texttt{OR}- and \texttt{HC}-methods are omitted in Figure \ref{fig:HDplots}, and the comparison between fractional pseudo-likelihood is presented in Figure \ref{fig:HDplots2}. The corresponding true positive and false positive rates for dimensions $d = 64$ and $d = 1024$ are presented in Table \ref{TPFP}. All the shown results are averages computed from 25 data sets. The \texttt{AND}- and \texttt{HC}-method maintain almost equally good performance regardless the dimension considered and obtain the best overall performance in terms of Hamming distances. The \texttt{OR}-method is better on smaller dimensions where the graph is denser in the relative sense.

\begin{figure}
\centerline{
\includegraphics[width= .99\linewidth]{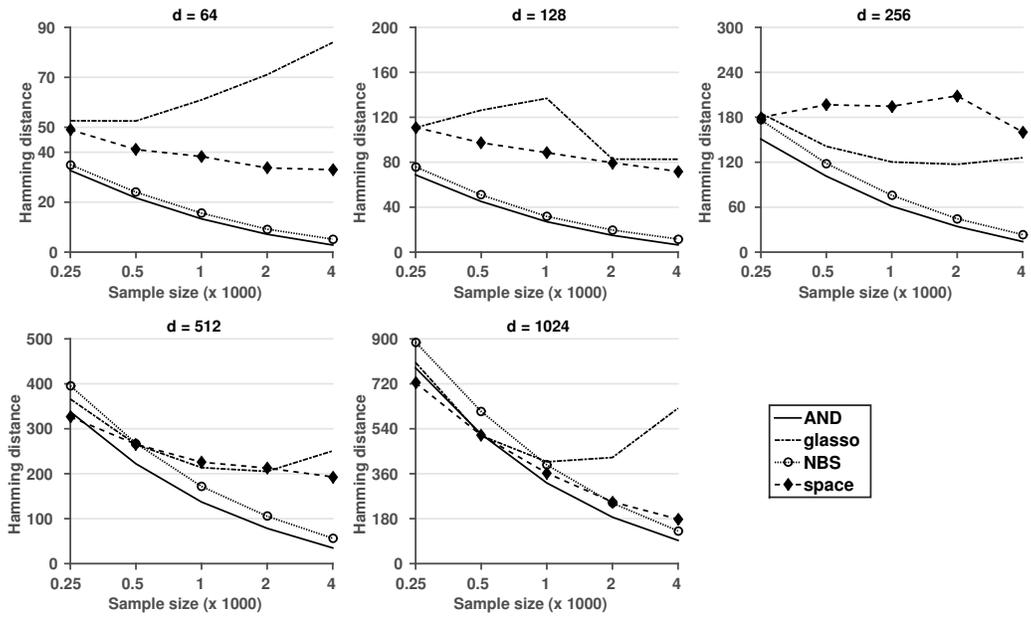}}\caption{Sample size versus Hamming distance plots. Dimensions considered are $p = 64, 128, 256, 512 \textnormal{ and } 1024$.}\label{fig:HDplots}
\end{figure}

\begin{figure}
\centerline{
\includegraphics[width= .99\linewidth]{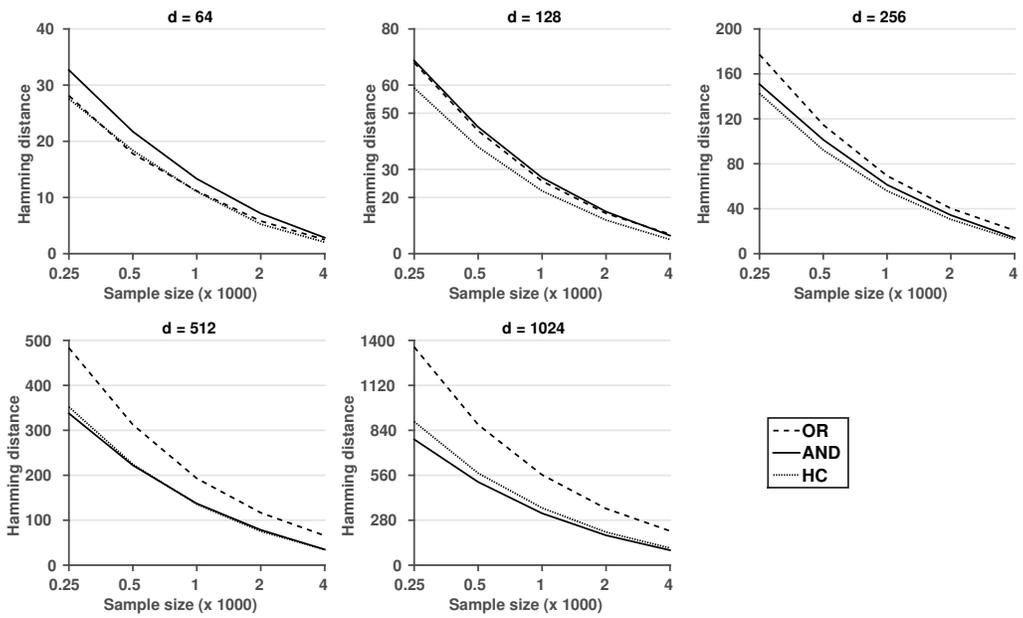}}\caption{Sample size versus Hamming distance plots for \texttt{OR}, \texttt{AND} and \texttt{HC}.}\label{fig:HDplots2}
\end{figure}

In the smaller dimensions \texttt{NBS} performs almost equally well as  \texttt{AND} and \texttt{HC}. The graphs estimated by \texttt{NBS} are really sparse resulting in a low false positive rate. The Hamming distance curves of \texttt{glasso} do not seem to decrease consistently as the sample size grows. We tried also using the ordinary BIC-criterion for choosing the tuning parameter for \texttt{glasso} but this resulted in  denser graphs and inferior Hamming distances (results not shown). The \texttt{space}-method improves its performance quite steadily as $n$ grows and has nearly always the best true positive rate. However, this comes with a cost in terms of the false positive rate which is always higher for \texttt{space} than for the best pseudo-likelihood method or \texttt{NBS}. When the sample size is less than the dimension, \texttt{space} achieves good results with Hamming distances being equal or slightly better than those of \texttt{AND}-method.

\begin{table}\caption{A table showing true positive (TP) and false positive (FP) rates for different methods and sample sizes when $p = 64, 1024$. For the full table with all the dimensions, see Appendix B.}
\begin{footnotesize}
\centerline{\begin{tabular}{c c | c c c c c c c c c c c c }
\multirow{2}{*}{$p$} & \multirow{2}{*}{$n$}& \multicolumn{2}{|c|}{\texttt{OR}}
& \multicolumn{2}{|c|}{\texttt{AND}}
& \multicolumn{2}{|c|}{\texttt{HC}}
& \multicolumn{2}{|c|}{\texttt{glasso}}
& \multicolumn{2}{|c|}{\texttt{NBS}}
& \multicolumn{2}{|c}{\texttt{space}}\\ 
& & \multicolumn{1}{|c}{TP} & \multicolumn{1}{c|}{FP}
& \multicolumn{1}{|c}{TP} & \multicolumn{1}{c|}{FP}
& \multicolumn{1}{|c}{TP} & \multicolumn{1}{c|}{FP}
& \multicolumn{1}{|c}{TP} & \multicolumn{1}{c|}{FP}
& \multicolumn{1}{|c}{TP} & \multicolumn{1}{c|}{FP}
& \multicolumn{1}{|c}{TP} & \multicolumn{1}{c}{FP}
\\
\hline
\multirow{5}{*}{$64$}&250 &0.72 &3e-03 &0.59 &4e-04 &0.68 &1e-03 &0.74 &2e-02 &0.57 &6e-04 &0.79 &2e-02 \\
&500 &0.81 &2e-03 &0.73 &2e-04 &0.78 &6e-04 &0.86 &2e-02 &0.71 &6e-04 &0.88 &2e-02 \\
&1000 &0.88 &1e-03 &0.83 &1e-04 &0.87 &4e-04 &0.93 &3e-02 &0.82 &9e-04 &0.94 &2e-02 \\
&2000 &0.95 &8e-04 &0.91 &6e-05 &0.94 &2e-04 &0.97 &4e-02 &0.90 &9e-04 &0.98 &2e-02 \\
&4000 &0.98 &4e-04 &0.96 &4e-05 &0.98 &1e-04 &0.99 &4e-02 &0.95 &8e-04 &0.99 &2e-02 \\
\hline
\multirow{5}{*}{$1024$}&250 &0.61 &2e-03 &0.50 &3e-04 &0.57 &7e-04 &0.39 &8e-05 &0.29 &1e-06 &0.56 &3e-04 \\
&500 &0.74 &1e-03 &0.66 &2e-04 &0.72 &4e-04 &0.66 &2e-04 &0.51 &2e-06 &0.72 &3e-04 \\
&1000 &0.85 &7e-04 &0.79 &1e-04 &0.83 &3e-04 &0.81 &3e-04 &0.68 &3e-06 &0.83 &3e-04 \\
&2000 &0.92 &5e-04 &0.88 &7e-05 &0.91 &2e-04 &0.91 &6e-04 &0.81 &3e-06 &0.91 &3e-04 \\
&4000 &0.97 &3e-04 &0.94 &5e-05 &0.96 &1e-04 &0.96 &1e-03 &0.90 &3e-06 &0.96 &2e-04 \\
\end{tabular}}\label{TPFP}
\end{footnotesize}
\bigskip
\end{table} 

To give a rough idea of the relative time complexity of the various methods, it took roughly half a second to estimate \texttt{OR}, \texttt{AND} and \texttt{HC} graphs in the $d = 64$ case when all the Markov blanket searches were run in a serial manner on a standard $2.3$ GHz workstation. The high-dimensional cases were solved in couple minutes. Average running times of the other methods are tabulated in Appendix B. To summarize, the \texttt{NBS}-method was clearly the fastest, whereas \texttt{space} took the longest to run. \texttt{Space} was generally fast to compute when $n$ was small but the running time varied considerably depending on the tuning parameter and grew quickly with the sample size. Even though computing a single instance of \texttt{glasso} or \texttt{space} might be faster than fractional pseudo-likelihood methods, one is usually forced to run these methods several times to find a suitable tuning parameter, thus making the actual running times much longer. Also, choosing an appropriate range for the candidate tuning parameters might prove difficult in some settings. These kind of practical difficulties make the method proposed here appealing, since no tuning parameters need to be chosen by the user. Furthermore, running the Markov blanket searches in parallel provides an easy improvement in efficiency.

\subsection{Brain Measurement Data}
We additionally illustrate the ability of the fractional marginal pseudo-likelihood to learn sparse structures by applying it to a real data set containing brain activity measurements. The whole data set consists of 2048 observations from a fMRI experiment on 90 variables corresponding to different regions of the brain. The data set is part of the R-package `brainwaver' by \cite{BRAINDATA}. 

We used the first 50 variables and fitted a first-order vector autoregressive model to remove the most significant dependencies between subsequent sample vectors. As a result, we obtain 2048 residual vectors that should by assumption follow a multivariate normal distribution. The obtained data was then split into a training set and a test set. The size of the test set was always taken to be $48$. For the training set size $m$, we considered three scenarios, where $m = 40$, $m=200$ or $m = 2000$.  Training data was always centered and scaled before applying methods. Centering of the test set was done using the means and standard deviations computed from the training data. 

For pseudo-likelihood methods and \texttt{NBS} we first learned the graphical structure and then computed the maximum likelihood estimate for the precision matrix given the structure. In case of \texttt{glasso} and \texttt{space}, the precision matrix is readily available from the output of the algorithm. In these experiments we considered also the case where the sparsity promoting graph prior was not used with pseudo-likelihood methods. 

For \texttt{glasso} we used 30 tuning parameters from the interval $[0.01 ,10]$, choosing the best according to the extended BIC criterion. The \texttt{space}-method was also computed with $30$ different tuning parameter values, scale selected as in the structure learning tests. Range of tuning parameters was again selected so that the best value according to the used BIC criterion would be strictly inside the grid. For \texttt{NBS} tuning parameter was chosen automatically as in the structure learning tests.

After the model learning, we took one data point at a time from the test set and tried to predict each of the components given the values of the others. Predicted value $\hat{X_i}$ for variable $x_i$ was computed as
$
\hat{X_i} = \sum_{j \not= i}\rho_{ij}\sqrt{{\omega_{jj}}/{\omega_{ii}}} X_j,  
$ where $\omega_{ii}$ are the diagonal elements of the estimated precision matrix and $\rho_{ij}$ are the partial correlations which can be obtained from the precision matrix. Squared difference of predicted value to the real value was recorded and the mean squared error (MSE) was used to compare the predictive performances of different methods. 

Table \ref{brainTable} shows the results of prediction tests for training sample sizes $m=40$, $m = 200$ and $m=2000$. Results for \texttt{AND} and \texttt{HC} methods are omitted, since these were generally slightly worse than the results of \texttt{OR}. The shown results are averages from 50 tests. We can observe that \texttt{OR} with a graph prior provides the lowest prediction error when the sample size is less than the dimension. When the number of observations grows, \texttt{OR} without prior obtains the best predictions. In general, the differences between the methods are quite marginal. However, the models estimated by \texttt{OR} are usually substantially sparser than the ones estimated by competing methods, especially with the highest sample size considered here, $m = 2000$. Sparse models are naturally a desirable results as they are easier to interpret. In addition to that, these conditional independences captured by \texttt{OR} are relevant in a sense that the corresponding model provides the lowest MSEs when predicting missing data.

\begin{table}\caption{A table showing average MSEs and edge densities (in parentheses) for different methods applied to brain data residuals.}
\begin{small}
\centerline{\begin{tabular}{c | c c c c c }
$m$ & \texttt{OR} & \texttt{ORprior} & \texttt{glasso} & \texttt{NBS} & \texttt{space} \\
\hline 40 & 1.002(11\%) & \textbf{0.968}(4\%) & 1.080(0\%) & 1.057(0\%) & 0.988(4\%)\\
200 & \textbf{0.713}(11\%) & 0.722(7\%) & 0.923(3\%) & 0.721(7\%) & 0.717(16\%)\\
2000 & \textbf{0.647}(22\%) & 0.650(16\%) & 0.648(34\%) & 0.650(23\%) & 0.649(32\%)\\
\end{tabular}}\label{brainTable}\bigskip \end{small}
\end{table}

We additionally studied the predictive performance using data from the same synthetic networks as in the structure learning tests. A data set of 2048 observations was created and the procedure used with the brain data was repeated. Size of the training set was $2000$ and the remaining $48$ observations formed the test set. Selection of tuning parameters for \texttt{glasso}, \texttt{space} and \texttt{NBS} was done as in the structure learning tests. Figure \ref{fig:SSEsynth} shows the MSE and the corresponding number of edges in the graphical model for dimensions $64$ and $128$. The shown results are averages computed from 25 data sets. Here, all the fractional marginal pseudo-likelihood based methods have slightly better prediction performances compared to other methods. 
\begin{figure}
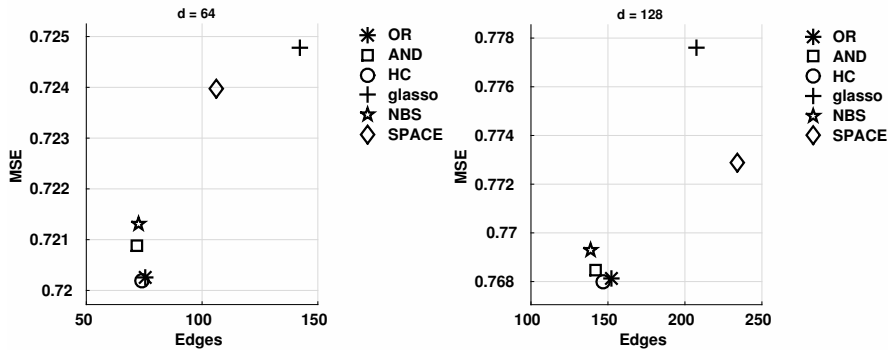

\centering
\includegraphics[width = .35\textwidth]{MSE64BW}
\includegraphics[width = .35\textwidth]{MSE128BW}\caption{MSE (vertical axis) and the number of found edges (horizontal axis) for the synthetic data.}\label{fig:SSEsynth}
\end{figure}

We emphasise that the results produced by our methods are achieved without needing to tune any hyperparameters. Only choice left to the user is whether to include the sparsity promoting prior or not. This can be a drastic advantage as demonstrated by the structure learning tests, where the suggested criteria for tuning parameter selection did not seem to be always optimal if the goal is to find the correct graphical structure. However, if the goal is to select a model with a good prediction power, the differences were not substantial and the used criteria produced good results.

\section{Discussion}
In this work we have introduced the fractional marginal pseudo-likelihood, an approximate Bayesian method for learning graphical models from multivariate normal data. One particular advantage of the method is its objectivity, since it does not necessitate the use of any domain specific knowledge which may be difficult to elicitate and use in general. In addition, the method allows graphs to be non-decomposable, which can be of substantial importance in applications. Earlier research has demonstrated that when the data generating process deviates from decomposability, graph learning methods building on the assumption of decomposability tend to yield unnecessarily dense graphs resulting from addition of spurious edges to chordless cycles. 

As shown formally, our method enjoys consistency and was found in simulation experiments to yield considerably more accurate estimates of the graph structure than the competing methods. For many applications of graphical model learning it is essential to retain solid interpretability of the estimated covariance structure, which means that high fidelity of the graph structure estimator is the most desirable property. In particular, frequently arising spurious edges may lead to confusing interpretations in the high-dimensional setting. In terms of predictive performance, methods considered here delivered similar levels of accuracy. Our divide-and-conquer type solution offers a possibility for efficient parallelization, as the initial Markov blanket search can be performed independently for each node. Hence, an attractive target for future research would include applications to very high-dimensional data sets and development of various parallelization schemes. In addition, it would be interesting to investigate altered versions by making the method more robust to outliers through relaxing the Gaussian assumption. The robust method could for instance be compared against a method by \cite{ROBUST}, which was shown to perform better than \texttt{glasso} when the data follow a heavier tailed distribution than a Gaussian.

\section*{Acknowledgements}
The research in this article was supported by Academy of Finland, grant no. 251170.
\vspace*{-8pt}

\bibliographystyle{apalike} 

\begin{thebibliography}{}

\bibitem[Achard, 2012]{BRAINDATA}
Achard, S. (2012).
\newblock R-package brainwaver: Basic wavelet analysis of multivariate time
  series with a visualisation and parametrisation using graph theory, version
  1.6.
\newblock http://cran.r-project.org/web/packages/brainwaver/.

\bibitem[Atay-Kayis and Massam, 2005]{ATAY05}
Atay-Kayis, A. and Massam, H. (2005).
\newblock A {M}onte {C}arlo method for computing the marginal likelihood in
  nondecomposable {G}aussian graphical models.
\newblock {\em Biometrika}, 92(2):317--335.

\bibitem[Besag, 1972]{BESAG1972}
Besag, J.~E. (1972).
\newblock Nearest-neighbour systems and the auto-logistic model for binary
  data.
\newblock {\em Journal of the Royal Statistical Society. Series B
  (Methodological)}, 34(1):75--83.

\bibitem[Carvalho and Scott, 2009]{CARVALHO09}
Carvalho, C.~M. and Scott, J.~G. (2009).
\newblock Objective {B}ayesian model selection in {G}aussian graphical models.
\newblock {\em Biometrika}, 96(3):497--512.

\bibitem[Consonni and Rocca, 2012]{CONSONNI12}
Consonni, G. and Rocca, L.~L. (2012).
\newblock Objective {B}ayes factors for {G}aussian directed acyclic graphical
  models.
\newblock {\em Scandinavian Journal of Statistics}, 39(4):743--756.

\bibitem[Dempster, 1972]{DEMPSTER72}
Dempster, A.~P. (1972).
\newblock Covariance selection.
\newblock {\em Biometrics}, 28(1):157--175.

\bibitem[Dobra et~al., 2004]{DOBRA04}
Dobra, A., Hans, C., Jones, B., Nevins, J.~R., Yao, G., and West, M. (2004).
\newblock Sparse graphical models for exploring gene expression data.
\newblock {\em Journal of Multivariate Analysis}, 90(1):196--212.
\newblock Special Issue on Multivariate Methods in Genomic Data Analysis.

\bibitem[Dobra et~al., 2011]{DOBRA11}
Dobra, A., Lenkoski, A., and Rodriguez, A. (2011).
\newblock {B}ayesian inference for general {G}aussian graphical models with
  application to multivariate lattice data.
\newblock {\em Journal of the American Statistical Association},
  106(496):1418--1433.

\bibitem[Fitch et~al., 2014]{FITCH14}
Fitch, A.~M., Jones, M.~B., and Massam, H. (2014).
\newblock The performance of covariance selection methods that consider
  decomposable models only.
\newblock {\em Bayesian Analysis}, 9(3):659--684.

\bibitem[Foygel and Drton, 2010]{EBIC}
Foygel, R. and Drton, M. (2010).
\newblock Extended {B}ayesian information criteria for {G}aussian graphical
  models.
\newblock In Lafferty, J., Williams, C., Shawe-taylor, J., Zemel, R., and
  Culotta, A., editors, {\em Advances in Neural Information Processing Systems
  23}, pages 604--612.

\bibitem[Friedman et~al., 2008]{GLASSO1}
Friedman, J., Hastie, T., and Tibshirani, R. (2008).
\newblock Sparse inverse covariance estimation with the graphical lasso.
\newblock {\em Biostatistics}, 9(3):432--441.

\bibitem[Friedman et~al., 2014]{GLASSOCODES}
Friedman, J., Hastie, T., and Tibshirani, R. (2014).
\newblock R-package glasso: Graphical lasso- estimation of {G}aussian graphical
  models, version 1.8.
\newblock http://cran.r-project.org/web/packages/glasso/.

\bibitem[Geiger and Heckerman, 2002]{GEIGER02}
Geiger, D. and Heckerman, D. (2002).
\newblock Parameter priors for directed acyclic graphical models and the
  characterization of several probability distributions.
\newblock {\em The Annals of Statistics}, 30(5):1412--1440.

\bibitem[Gelman et~al., 2014]{BAYESDATA}
Gelman, A., Carlin, J.~B., Stern, H.~S., Dunson, D.~B., Vehtari, A., and Rubin,
  D.~B. (2014).
\newblock {\em Bayesian Data Analysis}.
\newblock Chapman and Hall/CRC.

\bibitem[Haughton, 1988]{haugh88}
Haughton, D. M.~A. (1988).
\newblock On the choice of a model to fit data from an exponential family.
\newblock {\em Ann. Statist.}, 16(1):342--355.

\bibitem[Jones et~al., 2005]{JONES05b}
Jones, B., Carvalho, C., Dobra, A., Hans, C., Carter, C., and West, M. (2005).
\newblock Experiments in stochastic computation for high-dimensional graphical
  models.
\newblock {\em Statist. Sci.}, 20(4):388--400.

\bibitem[Koller and Friedman, 2009]{KOLLER}
Koller, D. and Friedman, N. (2009).
\newblock {\em Probabilistic Graphical Models: Principles and Techniques}.
\newblock MIT Press.

\bibitem[Lauritzen, 1996]{LAURITZEN}
Lauritzen, S. (1996).
\newblock {\em Graphical Models}.
\newblock Clarendon Press.

\bibitem[Meinshausen and B\"{u}hlmann, 2006]{MB06}
Meinshausen, N. and B\"{u}hlmann, P. (2006).
\newblock High-dimensional graphs and variable selection with the lasso.
\newblock {\em Ann. Statist.}, 34(3):1436--1462.

\bibitem[Moghaddam et~al., 2009]{MOGHADDAM09}
Moghaddam, B., Khan, E., Murphy, K.~P., and Marlin, B.~M. (2009).
\newblock Accelerating {B}ayesian structural inference for non-decomposable
  {G}aussian graphical models.
\newblock In Bengio, Y., Schuurmans, D., Lafferty, J., Williams, C., and
  Culotta, A., editors, {\em Advances in Neural Information Processing Systems
  22}, pages 1285--1293. Curran Associates, Inc.

\bibitem[O'Hagan, 1995]{OHAGAN95}
O'Hagan, A. (1995).
\newblock Fractional {B}ayes factors for model comparison.
\newblock {\em Journal of the Royal Statistical Society. Series B
  (Methodological)}, 57(1):99--138.

\bibitem[Peng et~al., 2009]{SPACE09}
Peng, J., Wang, P., Zhou, N., and Zhu, J. (2009).
\newblock Partial correlation estimation by joint sparse regression models.
\newblock {\em Journal of the American Statistical Association},
  104(486):735--746.
\newblock PMID: 19881892.

\bibitem[Peng et~al., 2010]{SPACECODES}
Peng, J., Wang, P., Zhou, N., and Zhu, J. (2010).
\newblock Sparse partial correlation estimation.
\newblock https://cran.r-project.org/web/packages/space/.

\bibitem[Pensar et~al., 2014]{MPL}
Pensar, J., Nyman, H., Niiranen, J., and Corander, J. (2014).
\newblock Marginal pseudo-likelihood learning of {M}arkov network structures.
\newblock {\em arXiv:1401.4988v2}.

\bibitem[Press, 1982]{PRESS}
Press, J.~S. (1982).
\newblock {\em Applied Multivariate Analysis: Using Bayesian and Frequentist
  Method of Inference}.
\newblock Robert E. Krieger Publishing Company.

\bibitem[Scott and Carvalho, 2008]{SCOTT08}
Scott, J.~G. and Carvalho, C.~M. (2008).
\newblock Feature-inclusion stochastic search for {G}aussian graphical models.
\newblock {\em Journal of Computational and Graphical Statistics},
  17(4):790--808.

\bibitem[Sun and Li, 2012]{ROBUST}
Sun, H. and Li, H. (2012).
\newblock Robust {G}aussian graphical modeling via l1 penalization.
\newblock {\em Biometrics}, 68(4):1197--1206.

\bibitem[Wei, 1992]{wei92}
Wei, C.~Z. (1992).
\newblock On predictive least squares principles.
\newblock {\em Ann. Statist.}, 20(1):1--42.

\bibitem[Whittaker, 1990]{WHITTAKER}
Whittaker, J. (1990).
\newblock {\em Graphical Models in Applied Multivariate Statistics}.
\newblock John Wiley \& Sons.

\bibitem[Witten et~al., 2011]{GLASSO2}
Witten, D.~M., Friedman, J.~H., and Simon, N. (2011).
\newblock New insights and faster computations for the graphical lasso.
\newblock {\em Journal of Computational and Graphical Statistics},
  20(4):892--900.

\bibitem[Wong et~al., 2003]{WONG03}
Wong, F., Carter, C.~K., and Kohn, R. (2003).
\newblock Efficient estimation of covariance selection models.
\newblock {\em Biometrika}, 90(4):809--830.

\end{thebibliography}
\begin{small}

\end{small}

\appendix
\numberwithin{equation}{section}
\numberwithin{theorem}{section}
\numberwithin{table}{section}
\section{Appendix: Consistency Proof}

This section contains the proofs of Lemmas $1$ and $2$ which together imply the consistency of our method as formulated in Theorem 2 and Corollary 1. We follow the same notation and the assumptions given in Theorem 2.

The following proposition found in \cite{WHITTAKER} is used in the proof.

\begin{theorem}\label{deviance}
 (Based on $6.7.1;\ p. \phantom{.} 179)$ Suppose the normal random vector $\bm{x}$ can be partitioned into three $(\bm{x}_A, \bm{x}_B, \bm{x}_C)$ and all conditional independence constraints can be summarised by the single statement $\bm{x}_B \independent \bm{x}_C \mid \bm{x}_A$. If $\bm{x}_A, \bm{x}_B$ and $\bm{x}_C$ are $p$-,$q$- and $r$-dimensional respectively, then the deviance
$$
\textnormal{dev}(\bm{x}_B \independent \bm{x}_C \mid \bm{x}_A)= -n \log \frac{|\bm{S}||\bm{S}_A|}{|\bm{S}_{A\cup B}||\bm{S}_{A\cup C}|}
$$has an asymptotic chi-squared distribution with $qr$ degrees of freedom.
\end{theorem}
Here $\bm{S}$ is defined as before, but in \citeauthor{WHITTAKER} $\bm{S}$ is used to denote the sample covariance matrix. It is clear that this does not change the statement of the theorem in any manner of consequence to our purposes. Note that theorem holds also if $A = \emptyset$, since complete independence can be considered a special case of the conditional independence. In this case, term $|\bm{S}_A|$ in the expression of deviance simply disappears.

\subsection{Overestimation (Lemma 1)}

Let $\mathit{mb}^* \subset V\setminus \{ j \}$ and $fa^* = \mathit{mb}^* \cup \{ j \} $ denote the true Markov blanket and the true family of the node $x_j$, respectively. We denote the cardinality of $\mathit{mb}^*$ by $p_j$. Let $\mathit{mb}\subset V \setminus \{ j \}$ be a superset of the true Markov blanket $\mathit{mb}^*$. Denote $a = |\mathit{mb}| - p_j$. Since $\mathit{mb}^* \subset \mathit{mb}$, we have $a > 0.$
 
We want to show that
$$
\log\frac{p(\bm{X}_j\ \mid\ \bm{X}_{\mathit{mb}^*})}{p(\bm{X}_j\ \mid\ \bm{X}_{\mathit{mb}})} \to \infty
$$ in probability, as $n\to\infty$. Showing this will guarantee that fractional marginal pseudo-likelihood prefers the true Markov blanket over its supersets as the sample size increases. Remember, that the local fractional marginal pseudo-likelihood for $\mathit{mb}(j)$ was given according to
$$
p(\bm{X}_j \mid \bm{X}_{\mathit{mb(j)}}) = \pi^{-\frac{(n-1)}{2}} \frac{\Gamma \Le \frac{n+p_j}{2} \Ri}{ \Gamma \Le \frac{p_j+1}{2} \Ri}
 n^{-\frac{2p_j+1}{2}}\Le \frac{|\bm{S}_{\mathit{fa(j)}}|}{|\bm{S}_{\mathit{mb(j)}}|}\Ri^{-\frac{n-1}{2}}.
$$

Consider next the log ratio of local fractional marginal pseudo-likelihoods, for $\mathit{mb}^*$ and $\mathit{mb}$. The term containing the power of $\pi$ appears in both of the terms, and so it cancels. By noticing that 
$$
{n^{-\Le\frac{1 + 2p_j}{2}\Ri}}\bigg/{ n^{-\Le\frac{1 + 2(p_j+a)}{2}\Ri}} = n^{a},
$$ we get the following form for the ratio 

\begin{eqnarray}\label{logscore}
\log\frac{p(\bm{X}_j \mid \bm{X}_{\mathit{mb}^*})}{p(\bm{X}_j \mid \bm{X}_{\mathit{mb}})} &=& \log\frac{\Gamma\Le\frac{n + p_j}{2}\Ri}{\Gamma\Le\frac{n + p_j+ a}{2}\Ri} + \log\frac{\Gamma\Le\frac{1+ p_j + a}{2}\Ri}{\Gamma\Le\frac{1 + p_j}{2}\Ri} \nonumber\\
&+& a\log n - \Le\frac{n-1}{2}\Ri\log \Le\frac{|\bm{S}_{fa^*}||\bm{S}_{\mathit{mb}}|}{|\bm{S}_{\mathit{mb}^*}||\bm{S}_{fa}|}\Ri .
\end{eqnarray}
The second term in (\ref{logscore}) doesn't depend on $n$ so it can be omitted when considering the leading terms as $n\to\infty$. Denote $m = (n + p_j)/2$. Clearly $m \to\infty$, as $n\to\infty.$ Now we can write the first term in (\ref{logscore}) as 
\begin{equation}\label{ekatermi}
\log\frac{\Gamma(m)}{\Gamma\Le m + \frac{a}{2} \Ri} = \log\Gamma(m) - \log\Gamma\Le m + \frac{a}{2} \Ri .
\end{equation} 
Now letting $n\to\infty$ and by using Stirling's asymptotic formula for each of the terms in (\ref{ekatermi}), we get
\begin{eqnarray*}
\log\Gamma(m) - \log\Gamma\Le m + \frac{a}{2} \Ri
&=& \Le m - \frac{1}{2} \Ri \log m - m \\
&-&\Le \Le m + \frac{a}{2} - \frac{1}{2} \Ri \log \Le m + \frac{a}{2}\Ri - \Le m + \frac{a}{2} \Ri  \Ri +O(1).
\end{eqnarray*}
We see that $m$-terms cancel and the constant $a/2$ in the second term can be omitted. After rearranging the terms, the result can be written as
\begin{equation}\label{approks}
m \log \Le \frac{m}{m + \frac{a}{2} } \Ri + \frac{1}{2}\log \Le \frac{ m + \frac{a}{2}}{m} \Ri - \frac{a}{2}   \log \Le m + \frac{a}{2}\Ri +O(1).
\end{equation}
As $n\to\infty$, we have that
\begin{eqnarray*}
m \log \Le \frac{m}{m + \frac{a}{2} } \Ri = \frac{1}{2}\log \Le \frac{1}{1 + \frac{a}{2m}} \Ri^{2m} \to \frac{1}{2} \log (\exp(-a)) = -\frac{a}{2}
\end{eqnarray*}
and
$$
 \frac{1}{2}\log \Le \frac{ m + \frac{a}{2}}{m} \Ri = \frac{1}{2}\log \Le 1 + \frac{a}{2m} \Ri \to 0.
$$
Thus, we can write (\ref{ekatermi}) asymptotically as
$$\label{ekalopullinen}
\log\frac{\Gamma(m)}{\Gamma\Le m + \frac{a}{2} \Ri} = -\frac{a}{2}\log \Le m + \frac{a}{2}\Ri + O(1),
$$
or equivalently by using variable $n$
\begin{equation}
\log\frac{\Gamma\Le\frac{n + p_j}{2}\Ri}{\Gamma\Le\frac{n + p_j+ a}{2}\Ri} = -\frac{a}{2}\log \Le \frac{n+ p_j + a}{2}\Ri + O(1).
\end{equation}
No we can simplify the original formula (\ref{logscore}) by combining the first and the third term
\begin{eqnarray}
\log\frac{\Gamma\Le\frac{n + p_j}{2}\Ri}{\Gamma\Le\frac{n + p_j+ a}{2}\Ri} + a \log n
&=& -\frac{a}{2}\log \Le \frac{n+ p_j + a}{2}\Ri + \frac{a}{2} \log n^2 + O(1) \nonumber\\
&=& \frac{a}{2} \log n + O(1).
\end{eqnarray}

Consider next the last term in (\ref{logscore})
\begin{equation}\label{determinanttitermi}
- \Le\frac{n-1}{2}\Ri\log \Le\frac{|\bm{S}_{fa^*}||\bm{S}_{\mathit{mb}}|}{|\bm{S}_{\mathit{mb}^*}||\bm{S}_{fa}|}\Ri .
\end{equation}
Since $\mathit{mb}^* \subset \mathit{mb},$ we can write $\mathit{mb} = \mathit{mb}^* \cup R,$ where $R$ denotes the set of redundant variables in $\mathit{mb}$. Recall the Theorem \ref{deviance} and notice that by denoting $$A = \mathit{mb}^*,\ B = \{j \} \textnormal{ and } C = R,$$ it holds that $\bm{x}_B \independent \bm{x}_C \mid \bm{x}_A,$ since $\mathit{mb}^*$ was the true Markov blanket of $x_j$. Note also that in this case $qr= 1\cdot a=a$. Now the deviance can be written as
\begin{equation}
\textnormal{dev}(x_j \independent \bm{x}_R \mid \bm{x}_{\mathit{mb}^*}) = -n \log \Le\frac{|\bm{S}_{fa}||\bm{S}_{\mathit{mb}^*}|}{|\bm{S}_{fa^*}||\bm{S}_{\mathit{mb}}|}\Ri , 
\end{equation}
which is essentially just the determinant term (\ref{determinanttitermi}) multiplied by a constant $-2$. Let us denote $D_n = \textnormal{dev}(x_j \independent \bm{x}_R \mid \bm{x}_{\mathit{mb}^*})$. The determinant term gets the following representation
\begin{eqnarray}
- \Le\frac{n-1}{2}\Ri\log \Le\frac{|\bm{S}_{fa^*}||\bm{S}_{\mathit{mb}}|}{|\bm{S}_{\mathit{mb}^*}||\bm{S}_{fa}|}\Ri &= -\frac{n}{2}\log \Le\frac{|\bm{S}_{fa^*}||\bm{S}_{\mathit{mb}}|}{|\bm{S}_{\mathit{mb}^*}||\bm{S}_{fa}|}\Ri + O_p(1)\nonumber \\
&= -\frac{D_n}{2} + O_p(1).
\end{eqnarray} The $O_p(1)$ error on the first line comes from omitting the term  $$\frac{1}{2}\log \Le\frac{|\bm{S}_{fa^*}||\bm{S}_{\mathit{mb}}|}{|\bm{S}_{\mathit{mb}^*}||\bm{S}_{fa}|}\Ri. $$ Asymptotically, it holds that $D_n \sim \chi^2_a$. In other words the sequence $(D_n)$ converges in distribution to a random variable $D$, where $D \sim \chi_a^2.$ Convergence in distribution implies that the sequence $(D_n)$ is bounded in probability, that is, $D_n = O_p(1)$ for all $n$.

Combining the above findings, asymptotically
$$
- \Le\frac{n-1}{2}\Ri\log \Le\frac{|\bm{S}_{fa^*}||\bm{S}_{\mathit{mb}}|}{|\bm{S}_{\mathit{mb}^*}||\bm{S}_{fa}|}\Ri = O_p(1).
$$
Adding the results together, we have shown that, as $n\to\infty$
\begin{equation}
\log\frac{p(\bm{X}_j\ |\ \bm{X}_{\mathit{mb}^*})}{p(\bm{X}_j\ |\ \bm{X}_{\mathit{mb}})} = \frac{a}{2} \log n  + O_p(1). 
\end{equation}
Now since $a > 0,$ then
$$
\log\frac{p(\bm{X}_j\ |\ \bm{X}_{\mathit{mb}^*})}{p(\bm{X}_j\ |\ \bm{X}_{\mathit{mb}})} \to \infty
$$ in probability, as $n\to\infty$.

\subsection{Underestimation (Lemma 2)}

Let $\mathit{mb}^{*}$ denote the true Markov blanket of node $x_j$ and $\mathit{mb} \subset \mathit{mb}^{*}$. Let $A \subset V \backslash fa^*.$ Remember that $fa^*$ was defined to be $\mathit{mb}^* \cup \{j\}.$ Note that $A$ could also be an empty set. We want to show that
$$\log\frac{p(\bm{X}_j\ |\ \bm{X}_{\mathit{mb}^*\cup A})}{p(\bm{X}_j\ |\ \bm{X}_{\mathit{mb}\cup A})} \to \infty
$$in probability, as $n\to\infty.$ Denote $|\mathit{mb}^*\cup A| = p_j$ and $a = |\mathit{mb}\cup A| - p_j$. Here $a < 0,$ since $\mathit{mb}$ is a subset of the true Markov blanket. We can now proceed similarly as in the overestimation part, and write the log ratio as
\begin{eqnarray}\label{logscore2}
\log\frac{p(\bm{X}_j \mid \bm{X}_{\mathit{mb}^*\cup A})}{p(\bm{X}_j \mid \bm{X}_{\mathit{mb}\cup A})} &=& \log\frac{\Gamma\Le\frac{n + p_j}{2}\Ri}{\Gamma\Le\frac{n + p_j+ a}{2}\Ri} + \log\frac{\Gamma\Le\frac{1+ p_j + a}{2}\Ri}{\Gamma\Le\frac{1 + p_j}{2}\Ri} \nonumber\\
&+& a\log n - \Le\frac{n-1}{2}\Ri\log \Le\frac{|\bm{S}_{fa^*\cup A}||\bm{S}_{\mathit{mb}\cup A}|}{|\bm{S}_{\mathit{mb}^*\cup A}||\bm{S}_{fa\cup A}|}\Ri .
\end{eqnarray}
The first three terms are just the same ones appearing in (\ref{logscore}), which allows us to write
\begin{eqnarray}\label{logscore3}
\log\frac{p(\bm{X}_j \mid \bm{X}_{\mathit{mb}^*\cup A})}{p(\bm{X}_j \mid \bm{X}_{\mathit{mb}\cup A})} =\frac{a}{2}\log n - \Le\frac{n-1}{2}\Ri\log \Le\frac{|\bm{S}_{fa^*\cup A}||\bm{S}_{\mathit{mb}\cup A}|}{|\bm{S}_{\mathit{mb}^*\cup A}||\bm{S}_{fa\cup A}|}\Ri + O(1) .
\end{eqnarray}

\noindent Consider next the determinant term
\begin{equation}\label{determinanttitermi2}
- \Le\frac{n-1}{2}\Ri\log \Le\frac{|\bm{S}_{fa^*\cup A}||\bm{S}_{\mathit{mb}\cup A}|}{|\bm{S}_{\mathit{mb}^*\cup A}||\bm{S}_{fa\cup A}|}\Ri .
\end{equation}
By the definition of $\bm{S}$, it is clear that
$$
\frac{\bm{S}}{n} = \hat{\bm{\Sigma}} = \frac{1}{n}\sum_{i = 1}^n \bm{X}_i^{T}\bm{X}_i ,
$$ where $\hat{\bm{\Sigma}}$ is the maximum likelihood estimate of the true covariance matrix. As $n$ approaches infinity, the maximum likelihood estimate converges in probability to the true covariance matrix $\bm{\Sigma}$.

Letting $n\to\infty$, we can write the argument of logarithm in (\ref{determinanttitermi2}) as
\begin{equation}\label{ratio}
\left({\frac{|\bm{\Sigma}_{fa^*\cup A}|}{|\bm{\Sigma}_{\mathit{mb}^*\cup A}|}}\right)\bigg/ \left({\frac{|\bm{\Sigma}_{fa\cup A}|}{|\bm{\Sigma}_{\mathit{mb}\cup A}|}}\right)
\end{equation}
We can simplify the numerator and denominator by noticing that $\bm{\Sigma}_{fa^* \cup A}$ can be partitioned as

\begin{eqnarray*}
\left(\begin{array}{cccc} \textnormal{var}(x_j) & \phantom{f} & \textnormal{cov}(x_j, \bm{x}_{\mathit{mb}^*\cup A}) & \phantom{f} \\
\textnormal{cov}(x_j, \bm{x}_{\mathit{mb}^*\cup A})^T & \phantom{s}  &\bm{\Sigma}_{\mathit{mb}^*\cup A} & \phantom{s}
\end{array}\right),
\end{eqnarray*} where $\textnormal{var}(x_j)$ is the variance of variable $x_j$, $\textnormal{cov}(x_j, \bm{x}_{\mathit{mb}^*\cup A})$ is a horizontal vector containing covariances between $x_j$ and each of the variables in set $\mathit{mb}^*\cup A$. Using basic results concerning determinants of a partitioned matrix (see, for instance, \citealt{PRESS}), we have
\begin{eqnarray*}
|\bm{\Sigma}_{fa^*\cup A}| &=& |\bm{\Sigma}_{\mathit{mb}^*\cup A}| \cdot (\textnormal{var}(x_j)- \textnormal{cov}(x_j, \ \bm{x}_{\mathit{mb}^*\cup A}) \Le\bm{\Sigma}_{\mathit{mb}^*\cup A}\Ri^{-1}\textnormal{cov}(x_j,\ \bm{x}_{\mathit{mb}^*\cup A})^T) \\
&=& |\bm{\Sigma}_{\mathit{mb}^*\cup A}| \cdot \Le \textnormal{var}(x_j) - \textnormal{var}(\hat{x}_j[\bm{x}_{\mathit{mb}^*\cup A}]) \Ri \\
&=& |\bm{\Sigma}_{\mathit{mb}^*\cup A}| \cdot \textnormal{var}\Le x_j \mid  \bm{x}_{\mathit{mb}^*\cup A}\Ri,
\end{eqnarray*}
where we have used $\hat{x}_j[\bm{x}_{\mathit{mb}^*\cup A}]$ to denote the linear least squares predictor of $x_j$ from variables in the set $\mathit{mb}^*\cup A$. The last equality follows from the definition of partial variance, which is the residual variance of $x_j$ after subtracting the variance based on linear least squares predictor $\hat{x}_j[\bm{x}_{\mathit{mb}^*\cup A}]$.
Using this, we get
$$
\displaystyle\frac{|\bm{\Sigma}_{fa^*\cup A}|}{|\bm{\Sigma}_{\mathit{mb}^*\cup A}|} = \textnormal{var}\Le x_j  \mid  \bm{x}_{\mathit{mb}^*\cup A}\Ri .
$$ 
Applying this also for the ratio of $|\bm{\Sigma}_{fa\cup A}|$ and $|\bm{\Sigma}_{\mathit{mb}\cup A}|$, lets us to write (\ref{ratio}) as
\begin{equation}\label{detsievennetty}
\left({\frac{|\bm{\Sigma}_{fa^*\cup A}|}{|\bm{\Sigma}_{\mathit{mb}^*\cup A}|}}\right)\bigg/ \left({\frac{|\bm{\Sigma}_{fa\cup A}|}{|\bm{\Sigma}_{\mathit{mb}\cup A}|}}\right) =
\frac{\textnormal{var}\Le x_j  \mid  \bm{x}_{\mathit{mb}^*\cup A}\Ri}
{\textnormal{var}\Le x_j  \mid  \bm{x}_{\mathit{mb}\cup A}\Ri} .
\end{equation}
The form (\ref{detsievennetty}) makes it easier to analyse the behaviour of the determinant term and we can write the log ratio in (\ref{logscore2}) as follows
\begin{eqnarray}\label{logscore4}
\log\frac{p(\bm{X}_j \mid \bm{X}_{\mathit{mb}^*\cup A})}{p(\bm{X}_j \mid \bm{X}_{\mathit{mb}\cup A})} =\frac{a}{2}\log n - \frac{n}{2}\log \frac{\textnormal{var}\Le x_j \mid \bm{x}_{\mathit{mb}^*\cup A}\Ri}
{\textnormal{var}\Le x_j \mid \bm{x}_{\mathit{mb}\cup A}\Ri} + O_p(1) .
\end{eqnarray}
By investigating (\ref{logscore4}), it is clear that consistency is achieved if we can show that
\begin{equation}\label{inequlity}
\frac{\textnormal{var}\Le x_j \mid \bm{x}_{\mathit{mb}^*\cup A}\Ri}
{\textnormal{var}\Le x_j \mid \bm{x}_{\mathit{mb}\cup A}\Ri} < 1.
\end{equation} The equation (\ref{inequlity}) is equivalent to
\begin{eqnarray}\label{inequality2}
\textnormal{var}\Le x_j \mid \bm{x}_{\mathit{mb}^*\cup A}\Ri &<& \textnormal{var}\Le x_j \mid \bm{x}_{\mathit{mb}\cup A}\Ri \nonumber \\
\Leftrightarrow \textnormal{var}(x_j) - \textnormal{var}(\hat{x}_j[\bm{x}_{\mathit{mb}^*\cup A}]) &<& \textnormal{var}(x_j) - \textnormal{var}(\hat{x}_j[\bm{x}_{\mathit{mb}\cup A}]) \nonumber \\
\Leftrightarrow \textnormal{var}(\hat{x}_j[\bm{x}_{\mathit{mb}^*\cup A}]) &>& \textnormal{var}(\hat{x}_j[\bm{x}_{\mathit{mb}\cup A}]).
\end{eqnarray} Now assume $\mathit{mb} \neq \emptyset$, and denote the missing true Markov blanket members by $R = \mathit{mb}^* \backslash \mathit{mb}$. Then by using the additivity of the explained variance (see \citeauthor{WHITTAKER} p.138), we can write the left side of (\ref{inequality2}) as
\begin{eqnarray*}
\textnormal{var}(\hat{x}_j[\bm{x}_{\mathit{mb}^*\cup A}]) &=&  \textnormal{var}(\hat{x}_j[\bm{x}_{\mathit{mb}\cup A \cup R}]) \\
& =& \textnormal{var}(\hat{x}_j[\bm{x}_{\mathit{mb}\cup A}]) + \textnormal{var}(\hat{x}_j[\bm{x}_R - \hat{\bm{x}}_R [\bm{x}_{\mathit{mb}\cup A}]]).
\end{eqnarray*}
The term $\textnormal{var}(\hat{x}_j[\bm{x}_R - \hat{x}_R [\bm{x}_{\mathit{mb}\cup A}]])> 0,$ since elements of $R$ are in $x_j's$ Markov blanket. This shows that (\ref{inequlity}) holds. 

If $\mathit{mb} = \emptyset,$ the inequality (\ref{inequality2}) can be written as
$$
\textnormal{var}(\hat{x}_j[\bm{x}_{\mathit{mb}^*\cup A}]) > \textnormal{var}(\hat{x}_j[\bm{x}_A]).
$$Using again the additivity of the explained variance, this becomes
$$
\textnormal{var}(\hat{x}_j[\bm{x}_A]) + \textnormal{var}(\hat{x}_j[\bm{x}_{\mathit{mb}^*} - \hat{x}_{\mathit{mb}^*}[\bm{x}_A]]) > \textnormal{var}(\hat{x}_j[\bm{x}_A]),
$$which clearly holds.

All in all, we have showed that
$$
- \frac{n}{2}\log \frac{\textnormal{var}\Le x_j \mid \bm{x}_{\mathit{mb}^*\cup A}\Ri}
{\textnormal{var}\Le x_j \mid \bm{x}_{\mathit{mb}\cup A}\Ri} \to \infty, 
$$in probability, as $n \to\infty.$ This implies that
$$\log\frac{p(\bm{X}_j \mid \bm{X}_{\mathit{mb}^*\cup A})}{p(\bm{X}_j \mid \bm{X}_{\mathit{mb}\cup A})} \to \infty
$$in probability, as $n\to\infty,$ since $n$ increases faster than $(a/2)\log n$ decreases.

\section{Appendix: Additional Numerical Results}
Table \ref{table1} contains results for all the considered dimensions in the structure learning tests with synthetic data.

\begin{table}\caption{A table showing true positive (TP) and false positive (FP) rates in structure learning tests for different methods and sample sizes.}\label{table1}\begin{footnotesize}
\centerline{\begin{tabular}{c c | c c c c c c c c c c c c }
\multirow{2}{*}{$p$} & \multirow{2}{*}{$n$}& \multicolumn{2}{|c|}{\texttt{OR}}
& \multicolumn{2}{|c|}{\texttt{AND}}
& \multicolumn{2}{|c|}{\texttt{HC}}
& \multicolumn{2}{|c|}{\texttt{glasso}}
& \multicolumn{2}{|c|}{\texttt{NBS}}
& \multicolumn{2}{|c}{\texttt{space}}\\
& & \multicolumn{1}{|c}{TP} & \multicolumn{1}{c|}{FP}
& \multicolumn{1}{|c}{TP} & \multicolumn{1}{c|}{FP}
& \multicolumn{1}{|c}{TP} & \multicolumn{1}{c|}{FP}
& \multicolumn{1}{|c}{TP} & \multicolumn{1}{c|}{FP}
& \multicolumn{1}{|c}{TP} & \multicolumn{1}{c|}{FP}
& \multicolumn{1}{|c}{TP} & \multicolumn{1}{c}{FP}
\\
\hline
\multirow{5}{*}{$64$}&250 &0.72 &3e-03 &0.59 &4e-04 &0.68 &1e-03 &0.74 &2e-02 &0.57 &6e-04 &0.79 &2e-02 \\
&500 &0.81 &2e-03 &0.73 &2e-04 &0.78 &6e-04 &0.86 &2e-02 &0.71 &6e-04 &0.88 &2e-02 \\
&1000 &0.88 &1e-03 &0.83 &1e-04 &0.87 &4e-04 &0.93 &3e-02 &0.82 &9e-04 &0.94 &2e-02 \\
&2000 &0.95 &8e-04 &0.91 &6e-05 &0.94 &2e-04 &0.97 &4e-02 &0.90 &9e-04 &0.98 &2e-02 \\
&4000 &0.98 &4e-04 &0.96 &4e-05 &0.98 &1e-04 &0.99 &4e-02 &0.95 &8e-04 &0.99 &2e-02 \\
\hline
\multirow{5}{*}{$128$}&250 &0.71 &3e-03 &0.58 &4e-04 &0.67 &1e-03 &0.72 &8e-03 &0.52 &2e-04 &0.77 &9e-03 \\
&500 &0.81 &2e-03 &0.72 &2e-04 &0.78 &5e-04 &0.85 &1e-02 &0.68 &1e-04 &0.87 &1e-02 \\
&1000 &0.88 &1e-03 &0.83 &1e-04 &0.87 &3e-04 &0.91 &2e-02 &0.81 &2e-04 &0.93 &1e-02 \\
&2000 &0.94 &6e-04 &0.91 &6e-05 &0.93 &1e-04 &0.93 &9e-03 &0.89 &2e-04 &0.97 &9e-03 \\
&4000 &0.98 &4e-04 &0.96 &6e-05 &0.97 &9e-05 &0.97 &1e-02 &0.94 &3e-04 &0.99 &9e-03 \\
\hline
\multirow{5}{*}{$256$}&250 &0.68 &2e-03 &0.56 &4e-04 &0.64 &9e-04 &0.52 &1e-03 &0.44 &3e-05 &0.68 &2e-03 \\
&500 &0.79 &2e-03 &0.70 &2e-04 &0.76 &5e-04 &0.71 &2e-03 &0.62 &3e-05 &0.82 &4e-03 \\
&1000 &0.88 &1e-03 &0.82 &1e-04 &0.85 &3e-04 &0.84 &2e-03 &0.76 &5e-05 &0.91 &5e-03 \\
&2000 &0.94 &6e-04 &0.90 &8e-05 &0.92 &2e-04 &0.92 &3e-03 &0.86 &6e-05 &0.96 &6e-03 \\
&4000 &0.98 &4e-04 &0.96 &5e-05 &0.97 &1e-04 &0.97 &4e-03 &0.93 &7e-05 &0.99 &5e-03 \\
\hline
\multirow{5}{*}{$512$}&250 &0.65 &2e-03 &0.53 &3e-04 &0.61 &8e-04 &0.48 &3e-04 &0.37 &9e-06 &0.61 &6e-04 \\
&500 &0.77 &1e-03 &0.69 &2e-04 &0.74 &5e-04 &0.69 &5e-04 &0.57 &1e-05 &0.76 &9e-04 \\
&1000 &0.86 &8e-04 &0.80 &1e-04 &0.84 &3e-04 &0.82 &8e-04 &0.73 &1e-05 &0.86 &1e-03 \\
&2000 &0.93 &5e-04 &0.89 &7e-05 &0.92 &2e-04 &0.91 &1e-03 &0.83 &1e-05 &0.93 &1e-03 \\
&4000 &0.97 &4e-04 &0.95 &4e-05 &0.97 &1e-04 &0.97 &2e-03 &0.91 &1e-05 &0.97 &1e-03 \\
\hline
\multirow{5}{*}{$1024$}&250 &0.61 &2e-03 &0.50 &3e-04 &0.57 &7e-04 &0.39 &8e-05 &0.29 &1e-06 &0.56 &3e-04 \\
&500 &0.74 &1e-03 &0.66 &2e-04 &0.72 &4e-04 &0.66 &2e-04 &0.51 &2e-06 &0.72 &3e-04 \\
&1000 &0.85 &7e-04 &0.79 &1e-04 &0.83 &3e-04 &0.81 &3e-04 &0.68 &3e-06 &0.83 &3e-04 \\
&2000 &0.92 &5e-04 &0.88 &7e-05 &0.91 &2e-04 &0.91 &6e-04 &0.81 &3e-06 &0.91 &3e-04 \\
&4000 &0.97 &3e-04 &0.94 &5e-05 &0.96 &1e-04 &0.96 &1e-03 &0.90 &3e-06 &0.96 &2e-04 \\
\end{tabular}}\end{footnotesize}
\end{table}

Average running times for the different methods in the structure learning tests are presented in Table \ref{tablew}. For the marginal pseudo-likelihood methods and \texttt{NBS}, the shown times are average values computed from $10$ tests. Note, that the result shown for the \texttt{HC}-method is the time it took to perform the hill-climb after the \texttt{OR}-graph was first estimated.  In each of the ten tests (with given sample size and dimension), \texttt{space} and \texttt{glasso} were computed using $12$ different values for the tuning parameters as explained in the paper. Shown results for these two methods are averages computed over different tests and also over different tuning parameter values. All the timing experiments were run in Matlab or R on a standard laptop with a $2.30$ GHz quad-core processor.

\begin{table}\caption{A table containing average running times in seconds for the considered methods in the structure learning tests.}\label{tablew}
\centerline{\begin{tabular}{c c | c c c c c }
$d$ &$n$& \texttt{OR/AND} & \texttt{HC} & \texttt{glasso} & \texttt{NBS} & \texttt{space} \\
\hline \hline
\multirow{3}{*}{$64$}&250 &0.452&0.059&0.021&0.002&0.045\\
&1000 &0.548&0.076&0.015&0.002&0.301\\
&4000 &0.563&0.090&0.009&0.002&1.341\\
\hline
\multirow{3}{*}{$128$}&250 &1.843&0.116&0.135&0.008&0.158\\
&1000 &2.064&0.151&0.091&0.010&1.082\\
&4000 &2.243&0.182&0.055&0.007&6.965\\
\hline
\multirow{3}{*}{$256$}&250 &7.476&0.249&1.184&0.043&0.766\\
&1000 &8.300&0.316&0.657&0.044&3.943\\
&4000 &8.905&0.370&0.423&0.047&39.491\\
\hline
\multirow{3}{*}{$512$}&250 &31.595&0.604&11.776&0.328&3.907\\
&1000 &33.838&0.684&5.234&0.337&22.877\\
&4000 &36.229&0.807&3.308&0.327&185.111\\
\end{tabular}}
\end{table}

\end{document}